\documentclass{article}

\usepackage[main,final]{neurips_2026}

\usepackage[utf8]{inputenc}
\usepackage[T1]{fontenc}
\usepackage{url}
\usepackage{graphicx}
\usepackage{booktabs}
\usepackage{xcolor}
\usepackage{algorithm}
\usepackage{algorithmic}
\usepackage{multirow}
\usepackage{siunitx}
\usepackage{makecell}
\usepackage{float}
\usepackage{amsmath}
\usepackage{amssymb}
\usepackage{mathtools}
\usepackage{amsthm}
\usepackage{microtype}

\usepackage{hyperref}
\usepackage[capitalize,noabbrev]{cleveref}

\graphicspath{{../}{../figs/}{figs/}{../assets/}{assets/}}

\newcommand{\maybeincludegraphics}[2][]{%
  \IfFileExists{#2}{%
    \includegraphics[#1]{#2}%
  }{%
    \IfFileExists{../#2}{%
      \includegraphics[#1]{../#2}%
    }{%
      \fbox{%
        \begin{minipage}[c][0.08\textheight][c]{0.90\linewidth}
          \centering
          Missing figure asset:
          \texttt{\detokenize{#2}}
        \end{minipage}%
      }%
    }%
  }%
}

\theoremstyle{plain}
\newtheorem{theorem}{Theorem}[section]
\newtheorem{proposition}[theorem]{Proposition}

\theoremstyle{definition}

\theoremstyle{remark}

\title{LinearARD: linear-memory Attention Distillation for RoPE Restoration}

\author{%
  \textbf{Ning Yang}$^{1,\dagger}$%
  \thanks{Corresponding author: Ning Yang
    (\texttt{ning.yang@ia.ac.cn}).}%
  \quad
  \textbf{Hengyu Zhong}$^{2,}$%
  \thanks{Ning Yang, Hengyu Zhong, and Wentao Wang
    contributed equally to this work.}%
  \quad
  \textbf{Wentao Wang}$^{3,}$\footnotemark[2]%
  \quad
  \textbf{Baoliang Tian}$^{4}$ \\
  \textbf{Yuan Zhou}$^{4}$\quad
  \textbf{Haijun Zhang}$^{5}$\quad
  \textbf{Jun Wang}$^{6}$ \\
  $^{1}$Institute of Automation, Chinese Academy of Sciences,
    Beijing, China \\
  $^{2}$Southwest University, Chongqing, China \\
  $^{3}$Dalian University of Technology, Dalian, Liaoning, China \\
  $^{4}$Tianjin University, Tianjin, China \\
  $^{5}$University of Science and Technology Beijing,
    Beijing, China \\
  $^{6}$University College London, London, United Kingdom
}

\begin{document}

\maketitle

\begin{abstract}
The extension of context windows in Large Language Models is typically facilitated by scaling positional encodings followed by Continued Pre-Training (CPT). While effective, this paradigm is notoriously data-hungry and computationally expensive, requiring massive long-text corpora to recalibrate the model to the shifted positional distribution. We propose LinearARD, a restoration distillation method that restores Rotary Position Embedding (RoPE)-scaled students through attention-structure consistency with a frozen native-RoPE teacher. Rather than next-token prediction or opaque hidden-state matching, LinearARD aligns row-wise distributions of dense $Q/Q$, $K/K$, and $V/V$ self-relation matrices from the final attention layer to directly supervise attention dynamics. To remove the quadratic memory bottleneck of $n \times n$ relation maps, we introduce a linear-activation-memory kernel that stores only per-token log-sum-exp statistics and recomputes logits in the backward pass to obtain exact Kullback-Leibler divergence gradients. Across LLaMA2-7B, LLaMA3-8B, and Mistral-7B-v0.1 extended to 32K context, LinearARD recovers 93.1\%/94.2\%/94.3\% of native short-context performance and achieves strong long-context robustness on RULER using only \textbf{4.25M} tokens---amounting to just 1.6\% (a $\sim$60$\times$ reduction) of the 256M-token budget required by state-of-the-art baselines. Under this severely constrained budget, CPT and LongReD remain near zero on RULER, demonstrating that relation-level restoration provides a substantially better efficiency-quality tradeoff. Our code is available at \url{https://github.com/gracefulning/LinearARD}.
\end{abstract}

\section{Introduction}

The practical utility of Large Language Models (LLMs) is increasingly defined by their long-context capabilities, which are essential for advanced tasks such as retrieval-augmented generation \cite{asai2024self}, multi-step agentic workflows \cite{yao2022react}, and document-level reasoning \cite{liu2024lost}. However, training LLMs from scratch on extended context windows is prohibitively expensive, which motivates the development of post-hoc context extension methods that leverage pretrained models \cite{chen2023pi, peng2024yarn, ding2024longrope,zhu2024poseefficientcontextwindow}.

A widely used family of techniques extends the context window at inference time by modifying Rotary Position Embeddings (RoPE)~\cite{su2021roformer}. These methods include position interpolation and related scaling schedules such as linear interpolation~\cite{chen2023pi}, NTK-aware scaling~\cite{ntk2023}, Yet another RoPE extension method (YaRN)\cite{peng2024yarn}, and LongRoPE\cite{ding2024longrope}. While effective for long-context inference, such scaling often degrades short-context accuracy. This is because altering the rotary frequency schedule shifts the relative positional relationships between tokens and disrupts the learned attention patterns, which consequently impairs performance on shorter sequences.

To mitigate this trade-off, various approaches have been proposed. Continued Pre-Training (CPT)~\cite{ke2023continual} adapts the model by further training on long-context data under the scaled configuration. However, this approach typically requires substantial computational resources and high-quality long-context corpora, which are often scarce, and it runs the risk of catastrophic forgetting regarding the original short-context capabilities. Consequently, restoration distillation was introduced in subsequent works~\cite{gu2023minillm,longred2025}, which treats the original model as a teacher and the RoPE-scaled model as a student, fine-tuning the student to match the teacher on sequences within the teacher's native length while retaining the expanded maximum context. Nevertheless, existing restoration methods primarily focus on hidden-state matching. Because hidden states are aggregated outputs, they provide only indirect constraints that fail to precisely rectify the fine-grained positional distortions in the attention mechanism, limiting the efficiency and accuracy of the restoration.

Aligning distributional quantities within the attention module presents a conceptually appealing solution to the limitations of indirect constraints. Since RoPE scaling acts directly on queries and keys, it fundamentally alters attention logits; therefore, directly supervising the resulting distributions offers a more precise path to restoration. However, this approach faces a significant system-level bottleneck. Attention-level objectives typically incur quadratic memory overhead, and materializing full attention maps for backpropagation quickly exhausts GPU memory even at moderate sequence lengths. This burden is further exacerbated in a distillation setting, where both teacher and student models must be maintained simultaneously. Consequently, prior attention distillation methods~\cite{mobilebert2020,minilm2020,jiao2020tinybert} have been typically restricted to short sequences. Although various targets have been investigated, such as attention maps~\cite{mobilebert2020,jiao2020tinybert}, relation matrices~\cite{minilm2020,minilmv2}, and output logits~\cite{hinton2015distilling,sanh2019distilbert}, scaling these to long contexts often necessitates selective or sparse objectives~\cite{seekr2024}, which sacrifices exact distribution matching.

Following the discussion on attention distillation and full attention maps, \textbf{LinearARD} is proposed. This approach targets the root cause of performance degradation by enforcing structural consistency on the dense self-relation matrices, specifically $Q/Q$, $K/K$, and $V/V$, within each attention head. To overcome the associated memory bottleneck, we design an IO-aware tiled distillation kernel with linear activation memory. By recomputing relation logits on-chip during the backward pass, the kernel evaluates exact Kullback-Leibler (KL) divergence and its gradients without materializing full probability matrices, thereby enabling high-fidelity structural distillation on long sequences.

Our contributions are summarized as follows:
\begin{itemize}
    \item We introduce an IO-aware gradient fusion kernel that computes the exact KL divergence with linear activation memory complexity $\mathcal{O}(n)$. By bypassing the quadratic memory bottleneck, this kernel enables direct structural supervision on ultra-long sequences, where standard methods would otherwise fail due to GPU memory exhaustion.
    
    \item We propose LinearARD, a framework that enforces structural consistency across $Q/Q, K/K,$ and $V/V$ self-relations. This objective directly rectifies the positional misalignments induced by RoPE scaling and stabilizes representation drift, demonstrating robust transfer across materially different scaling schemes, including Position Interpolation and NTK-aware scaling.
    
    \item Theoretical analysis shows that the proposed kernel scales linearly in activation memory (Theorem~\ref{thm:memory}) while remaining exactly equivalent to standard full-matrix backpropagation (Proposition~\ref{prop:exactness}), providing a foundation for high-fidelity distillation.
    
    \item Extensive evaluations on LLaMA2-7B, LLaMA3-8B, and Mistral-7B-v0.1 show that LinearARD attains strong long-context performance on both RULER (63.0/67.8/60.2) and LongBench-E (22.65/21.18/17.17) using only 4.25M training tokens, while retaining 93.1\%/94.2\%/94.3\% of native short-context performance. This corresponds to roughly a 60$\times$ token reduction relative to 256M-token restoration baselines.

\end{itemize}

\section{Related Work}
\label{sec:background}

\paragraph{Efficient Attention and Memory Optimization.}
The quadratic memory complexity of self-attention ($O(n^2)$) constitutes the primary bottleneck for long-context modeling. Seminal works like FlashAttention~\cite{flashattention2022,flashattention2_2023} and Ring Attention~\cite{liu2023ringattention} address this by tiling computations in GPU SRAM and utilizing online statistics to compute the attention output without materializing the full attention matrix. While these kernels reduce the memory footprint of standard forward and backward passes to linear scale ($O(n)$), they do not support the computation of distributional alignment objectives. In a distillation setting, minimizing the KL divergence between teacher and student distributions typically necessitates instantiating full $n \times n$ probability matrices to calculate loss gradients, which reintroduces the quadratic memory bottleneck. Our work bridges this gap by proposing an IO-aware kernel that fuses the KL divergence calculation into the backward pass. Unlike prior kernels that execute self-attention forward and backward passes, our kernel enables exact, linear-activation-memory structural supervision, making dense relation alignment feasible at long sequence lengths where it was previously memory-prohibitive.

\begin{figure*}[t]
    \centering
    \maybeincludegraphics[width=\textwidth]{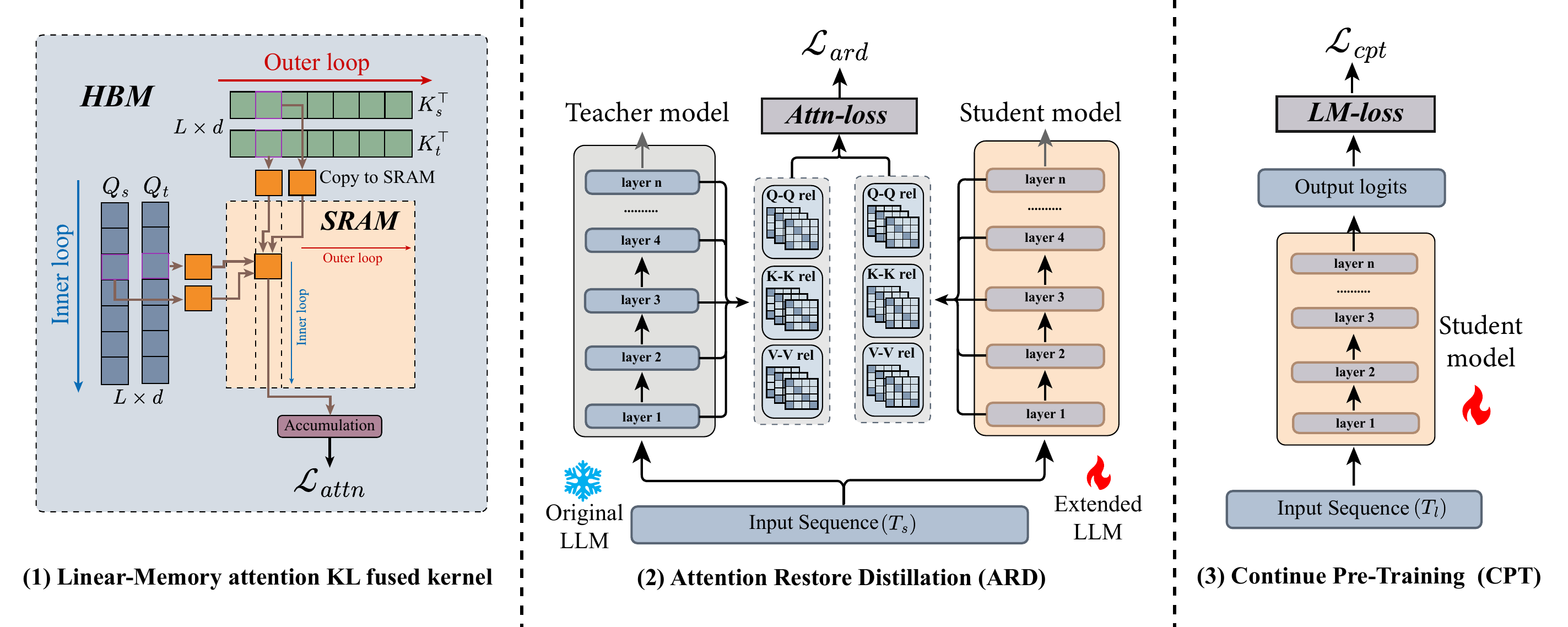}
    \caption{LinearARD pipeline. A frozen native-RoPE teacher provides final-layer dense row-wise relation distributions ($Q/Q$, $K/K$, and $V/V$), and the RoPE-scaled student is restored by minimizing relation KL with an exact linear-activation-memory kernel.}
    \label{fig:method_overview}
\end{figure*}

\paragraph{Structural Distillation in Transformers.}
Knowledge Distillation (KD)~\cite{hinton2015distilling} is a standard paradigm for transferring capabilities from a teacher to a student model. Beyond logit and hidden-state distillation, aligning attention mechanisms is critical for capturing linguistic dependencies. Approaches such as TinyBERT~\cite{jiao2020tinybert} and MobileBERT~\cite{mobilebert2020} align attention maps directly, while MiniLM~\cite{minilm2020,minilmv2} enhances stability by distilling self-relation modules. However, the efficacy of these methods is strictly limited by their spatial complexity; enforcing consistency on dense matrices is feasible only for short context windows. Consequently, recent long-context works resort to sparse supervision or token-level logit matching~\cite{seekr2024}, sacrificing structural fidelity for memory efficiency. By overcoming this memory barrier, our framework revisits and scales these dense objectives. Full-context alignment of $Q$, $K$, and $V$ self-relation distributions is enabled on ultra-long sequences, ensuring the student retains the precise structural priors of the teacher.

\paragraph{RoPE Scaling and Context Restoration.}
Techniques such as Position Interpolation (PI)\cite{chen2023pi}, NTK-aware scaling\cite{ntk2023}, YaRN~\cite{peng2024yarn}, and LongRoPE~\cite{ding2024longrope} extend the context window of pretrained LLMs by modifying RoPE frequencies or position indices. While effective for extension, this rescaling distorts the rotation-sensitive geometric relationships established during pretraining, leading to a collapse in short-context performance. CPT~\cite{ke2023continual} mitigates this but requires extensive compute and risks catastrophic forgetting. More recently, restoration distillation methods like LongReD~\cite{longred2025} have attempted to recover performance by matching hidden states. However, hidden states are aggregated features derived after the attention operation; they provide only a coarse, indirect signal that fails to isolate the root cause of the degradation. Since RoPE scaling directly perturbs the dot-product attention logits, restoration must target these internal relational structures directly~\cite{su2021roformer}. \textbf{LinearARD} corrects these fine-grained geometric distortions at their source, achieving significantly higher data efficiency and restoration accuracy than indirect hidden-state matching.

\section{Methodology}
\label{sec:method}

RoPE scaling modifies the rotary frequency schedule, altering the distance-dependent phase between queries and keys and causing the attention mechanism to drift relative to the original model. This redistribution of attention can degrade the short-context behavior established during pretraining, thereby reducing performance on short-context tasks. This issue is addressed by formulating restoration as a restoration distillation problem between an unscaled teacher model and a RoPE-scaled student model. The teacher provides stable supervision on sequences within its native context range, and the student is optimized to recover the teacher's short-context behavior while retaining the extended context window enabled by RoPE scaling.

This section first defines the restoration objective by specifying which internal structures are matched between the teacher and the RoPE-scaled student (Sec.~\ref{subsec:objective}). It then introduces an exact linear-activation-memory operator, a dedicated Kernel for KL distillation on dense $n\times n$ relation maps that bypasses the quadratic-memory bottleneck (Sec.~\ref{subsec:io_aware}). Finally, the discussion details the training procedure and implementation choices, including a parameter-efficient restoration recipe and an optional lightweight CPT stage (Sec.~\ref{subsec:training_recipe}).

\textbf{Method Overview.} As illustrated in Fig.~\ref{fig:method_overview}, the approach consists of three components in the above order. First, the teacher model's final-layer relational structure is distilled by aligning row-wise relation distributions induced by $Q/Q$, $K/K$, and $V/V$ self-relations (Sec.~\ref{subsec:objective}). Second, the quadratic-memory bottleneck is eliminated with an exact linear-activation-memory Kernel for KL distillation of dense $n\times n$ relation maps (Sec.~\ref{subsec:io_aware}). Third, a practical and parameter-efficient restoration procedure is adopted to optimize the student under these objectives, optionally augmented with a lightweight CPT stage (Sec.~\ref{subsec:training_recipe}).

\subsection{Problem Setup}
This work considers a decoder-only Transformer with $L$ layers and $H$ attention heads per layer. Let $B$ denote the batch size, $n$ the sequence length, and $d$ the per-head dimension. The original model with native RoPE serves as a frozen teacher, while the student shares the same architecture but employs a scaled RoPE configuration to support a larger maximum context length. During distillation, both models are evaluated on the same token sequence of length $n$ within the teacher's supported range. The teacher remains fixed to provide supervision, while the student uses scaled RoPE during both training and inference. Unless explicitly stated as an ablation, LinearARD applies the relation-alignment loss only to the final Transformer layer, denoted by $\ell^\star=L$.

The model employs masked causal self-attention with an additive mask $\mathbf{M}\in\mathbb{R}^{n\times n}$. Query positions are indexed by $i\in\{1,\ldots,n\}$ and key positions by $j\in\{1,\ldots,n\}$. This mask enforces causality and padding: $\mathbf{M}(i,j)=-\infty$ if position $j$ is not visible to query $i$, and $\mathbf{M}(i,j)=0$ otherwise.

For model index $m\in\{t,s\}$, where $t$ and $s$ denote the teacher and student respectively, the query, key, and value tensors in the selected final layer $\ell^\star$ are represented as $\mathbf{Q}_m \in \mathbb{R}^{B\times H_Q\times n\times d}$ and $\mathbf{K}_m, \mathbf{V}_m \in \mathbb{R}^{B\times H_{KV}\times n\times d}$, where $H_Q$ and $H_{KV}$ denote the number of query and key-value heads. To simplify the definition of the distillation objective, we formulate the loss for a single attention head and batch element, treating $\mathbf{Q}_m,\mathbf{K}_m,\mathbf{V}_m$ as matrices in $\mathbb{R}^{n\times d}$ and omitting head and batch indices when the context is unambiguous.

\begin{figure}[t]
    \centering
    \maybeincludegraphics[width=\textwidth]{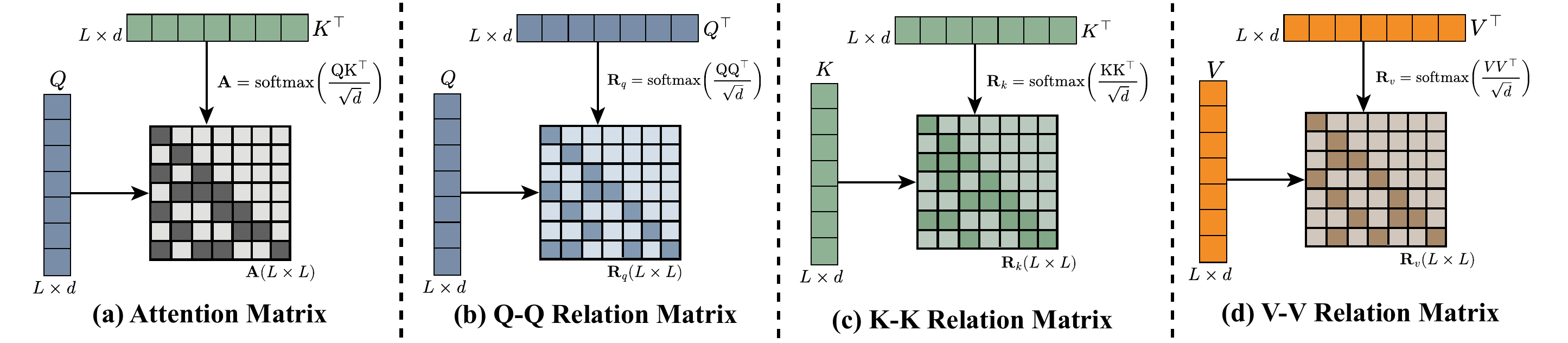}
    \caption{Final-layer standard attention vs.\ QKV self-relations. (a) Attention matrix $\mathbf{A}$ computed from $\mathbf{Q}\mathbf{K}^{\top}$ (shown for reference), and (b--d) Q/Q, K/K, and V/V relation matrices computed from $\mathbf{Q}\mathbf{Q}^{\top}$, $\mathbf{K}\mathbf{K}^{\top}$, and $\mathbf{V}\mathbf{V}^{\top}$ (Eq.~\ref{eq:qkv_relations}). LinearARD distills the row-wise relation distributions in (b--d) from the final Transformer layer by aligning teacher vs.\ student rows via forward KL (Eq.~\ref{eq:kl_def}) under the same mask $\mathbf{M}$.}
    \label{fig:qkv_relation}
\end{figure}

\subsection{Distillation Objectives}
\label{subsec:objective}
RoPE scaling perturbs the relative geometry of the projected representations, changing the relation distributions induced inside each attention head. The student is restored by matching these row-wise relation distributions to those of the frozen teacher.

\paragraph{Relation distributions.}
Fix the selected final layer, an attention head, and a batch element. Let $\mathbf{X}_m,\mathbf{Y}_m\in\mathbb{R}^{n\times d}$ be two sequences of projected vectors from model $m\in{t,s}$. A masked similarity logit matrix is defined
\begin{equation}
\label{eq:sim_logits_def}
    \mathbf{Z}_m \;=\; \frac{1}{\sqrt{d}}\,\mathbf{X}_m \mathbf{Y}_m^\top \;+\; \mathbf{M}
    \quad\in\mathbb{R}^{n\times n}.
\end{equation}
Applying a row-wise softmax yields a row-wise relation distribution
\begin{equation}
\label{eq:rel_dist_def}
    \mathbf{R}_m \;=\; \mathrm{softmax}(\mathbf{Z}_m),
\end{equation}
where $\mathbf{R}_m(i,:)$ is a categorical distribution over key positions $j\in\{1,\ldots,n\}$ for each fixed query position $i$.

For a fixed query position $i$, let $\mathcal{J} \subseteq \{1,\ldots,n\}$ denote the visible key positions. For each $j\in\mathcal{J}$, write $r_{t,j} \triangleq \mathbf{R}_t(i,j)$ and $r_{s,j} \triangleq \mathbf{R}_s(i,j)$ for the teacher and student probabilities, respectively, and $z_{s,j} \triangleq \mathbf{Z}_s(i,j)$ for the corresponding student logit.

\paragraph{KL objective.}
The distillation objective matches teacher and student relation distributions by minimizing the average forward KL divergence between corresponding rows:
\begin{equation}
\label{eq:kl_def}
    \mathcal{L}_{\mathrm{KL}}(\mathbf{R}_t,\mathbf{R}_s)
    = \frac{1}{n} \sum_{i=1}^n D_{\mathrm{KL}}\!\big(\mathbf{R}_t(i, :) \parallel \mathbf{R}_s(i, :)\big),
\end{equation}
In practice, this loss is computed on the final Transformer layer $\ell^\star=L$ and averaged over heads and batch elements; lower layers are not directly aligned in the default LinearARD objective. Layer-scope variants are studied only as ablations in Appendix~\ref{app:layerwise}.

\begin{proposition}[Gradient Behavior in Sparse Regimes]
\label{prop:grad_sensitivity}
Consider a query row over visible positions $\mathcal{J}$. Let $r_{t,j}$ and $r_{s,j} = \frac{\exp(z_{s,j})}{\sum_{a\in\mathcal{J}}\exp(z_{s,a})}$ denote the teacher and student distributions, with row losses:
\[
\mathcal{L}_{\mathrm{MSE}}^{\mathrm{row}} = \frac{1}{2}\sum_{j\in\mathcal{J}}(r_{s,j}-r_{t,j})^2,
\qquad
\mathcal{L}_{\mathrm{KL}}^{\mathrm{row}} = \sum_{j\in\mathcal{J}}r_{t,j}\log\frac{r_{t,j}}{r_{s,j}}.
\]
For any $k\in\mathcal{J}$, the derivatives with respect to logit $z_{s,k}$ satisfy:
\[
\frac{\partial\mathcal{L}_{\mathrm{MSE}}^{\mathrm{row}}}{\partial z_{s,k}} = r_{s,k}\Big[(r_{s,k}-r_{t,k}) - \sum_{j\in\mathcal{J}}r_{s,j}(r_{s,j}-r_{t,j})\Big],
\qquad
\frac{\partial\mathcal{L}_{\mathrm{KL}}^{\mathrm{row}}}{\partial z_{s,k}} = r_{s,k}-r_{t,k}.
\]
Consequently, as $r_{s,k}\to 0$ with fixed $r_{t,k} > 0$:
\[
\frac{\partial\mathcal{L}_{\mathrm{MSE}}^{\mathrm{row}}}{\partial z_{s,k}} \to 0,
\qquad
\frac{\partial\mathcal{L}_{\mathrm{KL}}^{\mathrm{row}}}{\partial z_{s,k}} \to -r_{t,k}.
\]
\end{proposition}
The proof is provided in Appendix~\ref{app:proof_gradient}.

This analytical distinction motivates the choice of forward KL over probability MSE because attention-like distributions are typically sparse and peaked. When the student assigns near-zero mass to a dependency supported by the teacher, MSE gradients can vanish due to the softmax Jacobian, whereas forward KL provides a first-order correction signal.

\begin{algorithm}[tb]
   \caption{Linear-Activation-Memory KL Distillation for Relation Distributions (Forward)}
   \label{alg:shadow_attn_fwd}
\begin{algorithmic}[1]
   \STATE {\bfseries Input:} Student $(\mathbf{X}_s,\mathbf{Y}_s)$, teacher $(\mathbf{X}_t,\mathbf{Y}_t)$, and $\mathcal{O}(n)$ mask metadata $\mathcal{C}$.
   \STATE {\bfseries Output:} Loss $\mathcal{L}$.
   \STATE {\bfseries Note:} For QKV relations, set $(\mathbf{X}_m,\mathbf{Y}_m)\leftarrow(\mathbf{Q}_m,\mathbf{Q}_m)$, $(\mathbf{K}_m,\mathbf{K}_m)$, or $(\mathbf{V}_m,\mathbf{V}_m)$ as in Eq.~\ref{eq:qkv_relations}.
   \STATE {\bfseries Note:} $\operatorname{ComputeLSE}(\mathbf{X}_m,\mathbf{Y}_m;\mathcal{C})$ computes $\mathrm{LSE}_m(a)=\log\sum_{b\in\mathcal{J}_a}\exp(\mathbf{Z}_m(a,b))$ via masked tiling.
   \STATE \textsc{Phase 1: Global Statistics (Linear Activation Memory)}
   \STATE $\mathrm{LSE}_s\leftarrow\operatorname{ComputeLSE}(\mathbf{X}_s,\mathbf{Y}_s;\mathcal{C})$
   \STATE $\mathrm{LSE}_t\leftarrow\operatorname{ComputeLSE}(\mathbf{X}_t,\mathbf{Y}_t;\mathcal{C})$
   \STATE \textsc{Phase 2: Fused Forward Pass via Tiling}
   \STATE Initialize loss $\mathcal{L}\leftarrow 0$
   \FOR{query blocks $\mathbf{X}_s^{(I)},\mathbf{X}_t^{(I)}$ loaded to SRAM}
       \FOR{key blocks $\mathbf{Y}_s^{(J)},\mathbf{Y}_t^{(J)}$ loaded to SRAM}
           \STATE \textcolor{gray}{// Generate mask tile on-the-fly in SRAM}
           \STATE $\mathbf{M}^{(I,J)}\leftarrow\operatorname{GenerateMask}(I,J;\mathcal{C}), \quad \mathcal{V}\leftarrow\{(r,c)\mid\mathbf{M}^{(I,J)}(r,c)=0\}$
           \STATE $\mathbf{Z}_s\leftarrow\frac{1}{\sqrt{d}}\mathbf{X}_s^{(I)}(\mathbf{Y}_s^{(J)})^\top+\mathbf{M}^{(I,J)}$
           \STATE $\mathbf{Z}_t\leftarrow\frac{1}{\sqrt{d}}\mathbf{X}_t^{(I)}(\mathbf{Y}_t^{(J)})^\top+\mathbf{M}^{(I,J)}$
           \STATE \textcolor{gray}{// Exclude masked entries before difference to avoid NaN}
           \STATE $\boldsymbol{\ell}_s\leftarrow\big(\mathbf{Z}_s-\mathrm{LSE}_s^{(I)}\big)[\mathcal{V}], \quad \boldsymbol{\ell}_t\leftarrow\big(\mathbf{Z}_t-\mathrm{LSE}_t^{(I)}\big)[\mathcal{V}]$
           \STATE $\mathbf{p}_t\leftarrow\exp(\boldsymbol{\ell}_t)$
           \STATE $\mathcal{L}\leftarrow\mathcal{L}+\sum\big[\mathbf{p}_t\odot(\boldsymbol{\ell}_t-\boldsymbol{\ell}_s)\big]$
       \ENDFOR
   \ENDFOR
   \STATE $\mathcal{L}\leftarrow\mathcal{L}/n$
\end{algorithmic}
\end{algorithm}

\paragraph{QKV self-relation targets.}
Eq.~\ref{eq:sim_logits_def} and Eq.~\ref{eq:rel_dist_def} are instantiated using $Q/Q$, $K/K$, and $V/V$ self-relations. Concretely, for each model $m\in\{t,s\}$ the following are defined:
\begin{subequations}\label{eq:qkv_relations}
\begin{align}
    \mathbf{R}^{Q}_m &= \mathrm{softmax}\!\left(\frac{\mathbf{Q}_m \mathbf{Q}_m^\top}{\sqrt{d}} + \mathbf{M}\right), \\
    \mathbf{R}^{K}_m &= \mathrm{softmax}\!\left(\frac{\mathbf{K}_m \mathbf{K}_m^\top}{\sqrt{d}} + \mathbf{M}\right), \\
    \mathbf{R}^{V}_m &= \mathrm{softmax}\!\left(\frac{\mathbf{V}_m \mathbf{V}_m^\top}{\sqrt{d}} + \mathbf{M}\right).
\end{align}
\end{subequations}
While $V/V$ is not a direct RoPE correction, it preserves the geometry of aggregated content and limits value drift under joint QKV training (see ablation in Table~\ref{tab:ablation_ard_components}, Appendix~\ref{app:ablation_details}). Although self-relations do not algebraically reconstruct cross-attention $\mathbf{Q}\mathbf{K}^\top$, minimizing them empirically restores attention dynamics with higher stability than direct attention matching. Fig.~\ref{fig:qkv_relation} visualizes the standard attention matrix alongside the $Q/Q$, $K/K$, and $V/V$ self-relation distributions.

\paragraph{Overall loss.}
Eq.~\ref{eq:kl_def} is applied to each of $\mathbf{R}^{Q}$, $\mathbf{R}^{K}$, and $\mathbf{R}^{V}$:
\begin{equation}
\label{eq:overall_loss}
\begin{split}
    \mathcal{L} ={} & \lambda_{q}\,\mathcal{L}_{\mathrm{KL}}(\mathbf{R}^{Q}_t,\mathbf{R}^{Q}_s) + \lambda_{k}\,\mathcal{L}_{\mathrm{KL}}(\mathbf{R}^{K}_t,\mathbf{R}^{K}_s)  + \lambda_{v}\,\mathcal{L}_{\mathrm{KL}}(\mathbf{R}^{V}_t,\mathbf{R}^{V}_s).
\end{split}
\end{equation}
Here $\lambda_q,\lambda_k,\lambda_v \ge 0$ are scalar weights; unless otherwise noted, $\lambda_q=\lambda_k=\lambda_v=1$ is used. The loss in Eq.~\ref{eq:overall_loss} is evaluated only for the selected final layer $\ell^\star$ in the main experiments.

\subsection{linear-activation-memory KL Distillation Kernel}
\label{subsec:io_aware}

Relational distillation requires matching dense $n\times n$ relation maps, whose naive implementation has quadratic activation memory. An IO-aware tiled gradient fusion kernel is proposed to compute the exact KL loss and gradients using $\mathcal{O}(n)$ memory by avoiding materialization of any $n\times n$ probability matrix.

\paragraph{Key identity.}
Let $\mathbf{Z}_s,\mathbf{Z}_t\in\mathbb{R}^{n\times n}$ be the masked similarity logits of the student and teacher, and let $\mathbf{R}_s=\mathrm{softmax}(\mathbf{Z}_s)$ and $\mathbf{R}_t=\mathrm{softmax}(\mathbf{Z}_t)$ be the corresponding row-wise distributions. Differentiating Eq.~\ref{eq:kl_def} with respect to the student logits yields
\begin{equation}
\label{eq:kl_grad_logits}
    \frac{\partial \mathcal{L}_{\mathrm{KL}}}{\partial \mathbf{Z}_s(i,j)} = \frac{1}{n}\Big(\mathbf{R}_s(i,j) - \mathbf{R}_t(i,j)\Big).
\end{equation}
Eq.~\ref{eq:kl_grad_logits} shows that each gradient entry depends only on the local probabilities for the same $(i,j)$. Therefore, if $\mathbf{R}_m(i,j)$ can be reconstructed on the fly without storing the full matrix, exact gradients can be computed with linear activation memory.

\paragraph{Two-pass tiled execution.}
A two-pass tiled strategy inspired by FlashAttention~\cite{flashattention2_2023} is adopted. First, the row-wise log-sum-exp statistics are computed and stored
\[
\mathrm{LSE}_m(i) = \log \sum_{k=1}^n \exp(\mathbf{Z}_m(i,k)),
\]
for both models $m\in\{t,s\}$. This requires $\mathcal{O}(n)$ storage in High Bandwidth Memory (HBM).

Second, query and key tiles of size $T_r\times T_c$ are iterated. For each tile, $\mathbf{Z}_s$ and $\mathbf{Z}_t$ are recomputed and local probabilities are reconstructed using the stored $\mathrm{LSE}_m(i)$ values:
\[
\mathbf{R}_m(i,j)=\exp\!\big(\mathbf{Z}_m(i,j)-\mathrm{LSE}_m(i)\big).
\]
All intermediate tile tensors reside in fast on-chip SRAM, and the accumulated gradients are streamed to HBM. Optionally, the scalar KL loss is computed in the same tiled pass via the decomposition
\begin{equation}
\label{eq:kl_decomposed}
\begin{aligned}
    \mathcal{L}_{\mathrm{KL}} = \frac{1}{n} \sum_{i,j} \mathbf{R}_t(i,j) \Big[ &(\mathbf{Z}_t(i,j) - \mathrm{LSE}_t(i)) -(\mathbf{Z}_s(i,j) - \mathrm{LSE}_s(i)) \Big].
\end{aligned}
\end{equation}

\paragraph{Complexity and guarantees.}
The efficiency and correctness of the kernel are established through the following statements.
\begin{theorem}[Linear Activation Memory Complexity]
\label{thm:memory}
The proposed KL distillation Kernel reduces activation memory complexity from $\mathcal{O}(B H n^2)$ to $\mathcal{O}(B H n d)$, which scales linearly with sequence length $n$ for fixed $d$.
\end{theorem}

\begin{proposition}[Mathematical Exactness]
\label{prop:exactness}
The gradients computed by the tiled formulation in Algorithm~\ref{alg:shadow_attn} (Appendix~\ref{app:backward_kernel}) are mathematically equivalent to the analytical gradients obtained from standard full-matrix backpropagation.
\end{proposition}

While activation memory scales as $\mathcal{O}(B H n d)$, arithmetic complexity remains quadratic ($\mathcal{O}(B H n^2 d)$) since all pairwise relations are evaluated on-chip. The proofs for Theorem~\ref{thm:memory} and Proposition~\ref{prop:exactness} are provided in Appendix~\ref{app:complexity} and Appendix~\ref{app:exactness}. Figure~\ref{fig:linear_kl_memory} provides empirical diagnostics verifying memory scaling and numerical exactness, and Appendix~\ref{app:kernel_throughput} reports wall-clock throughput. By choosing $(\mathbf{X},\mathbf{Y})$ as $(\mathbf{Q},\mathbf{Q})$, $(\mathbf{K},\mathbf{K})$, or $(\mathbf{V},\mathbf{V})$, this Kernel enables exact QKV relation distillation described in Sec.~\ref{subsec:objective} for long sequences.

\begin{figure}[t]
    \centering
    \maybeincludegraphics[width=0.74\linewidth]{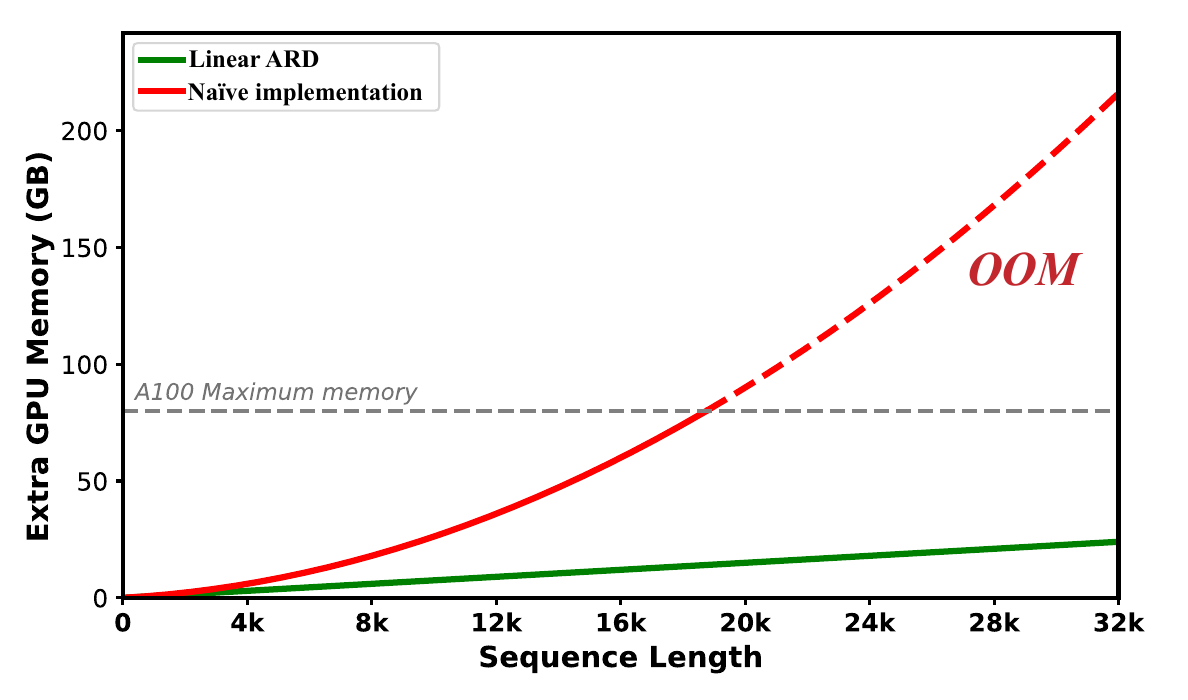}
    \caption{Memory scaling of the linear-activation-memory KL distillation Kernel in Sec.~\ref{subsec:io_aware} as a function of sequence length.}
    \label{fig:linear_kl_memory}
\end{figure}

\subsection{Training Recipe and Parameter Efficiency}
\label{subsec:training_recipe}
With the distillation targets and losses defined and the exact linear-activation-memory KL kernel established, the training procedure used to optimize the RoPE-scaled student is described next. The goal is to restore short-context behavior while retaining the extended context window enabled by RoPE scaling.
\paragraph{Stage 1: Attention distillation.}
Distillation is performed on sequences within the teacher's native context range. By optimizing Eq.~\ref{eq:overall_loss} on the final attention layer, the $Q/Q$, $K/K$, and $V/V$ relation distributions in Eq.~\ref{eq:qkv_relations} are aligned, directly correcting relational distortions introduced by RoPE scaling while avoiding the cost of supervising every layer.

\paragraph{Stage 2: Context adaptation.}
After relation distillation, a lightweight CPT stage is optionally run under the scaled RoPE configuration. This stage uses the standard language modeling objective to adapt the student to the expanded context window. In this setting, Stage~1 already restores short-context behavior, so this optional stage can be kept compute-efficient.

\paragraph{Parameter efficiency.}
Most of the model is frozen and only the attention projection matrices are optimized, focusing updates on components most sensitive to RoPE perturbations. To further reduce training cost, LoRA or QLoRA adapters~\cite{hu2022lora,dettmers2023qlora} are optionally employed during both stages.

\section{Experiments}
\label{sec:experiments}

\subsection{Experimental Setup}
\label{subsec:exp_setup}

\paragraph{Models.}
We study LLaMA2-7B~\cite{touvron2023llama}, LLaMA3-8B~\cite{grattafiori2024llama}, and Mistral-7B-v0.1, using native training windows of 4K, 8K, and 8K, respectively. The main experiments extend these windows to 32K using Position Interpolation (PI)~\cite{chen2023pi}: the scaling factor is eight for LLaMA2 and four for LLaMA3 and Mistral. LLaMA2 uses 32 query and 32 key/value heads, while the two GQA backbones use 32 query and eight key/value heads. To test sensitivity beyond PI, we additionally evaluate fixed-theta NTK-aware scaling~\cite{ntk2023} on LLaMA2-7B. The selected positional configuration is held fixed during restoration, adapter merging, CPT, and evaluation; the teacher retains native RoPE. Appendix~\ref{app:scaling_transfer} specifies the exact PI and NTK configurations.

\begin{table*}[t]
\centering
\scriptsize
\renewcommand{\arraystretch}{1.10}
\setlength{\tabcolsep}{2.0pt}
\resizebox{\textwidth}{!}{%
\begin{tabular}{l c c l r r r r r r r r r r r r}
\toprule
\multicolumn{4}{c}{} & \multicolumn{9}{c}{Short-text accuracy (\%)} & \multicolumn{3}{c}{Summary} \\
\cmidrule(lr){5-13} \cmidrule(lr){14-16}
Model & CW & PE & Method & MMLU & LAMB. & MathQA & BoolQ & OBQA & PIQA & SIQA & ARC-C & Avg. & Rec. & RULER & Tokens \\
\midrule
\multirow{8}{*}{LLaMA2-7B} & 4K & -- & -- & 45.99 & 71.18 & 29.92 & 78.23 & 44.00 & 78.73 & 40.28 & 49.23 & 54.69 & 100.0 & -- & -- \\
 & 32K & PI $8\times$ & -- & 24.65 & 6.18 & 19.73 & 58.29 & 36.20 & 69.31 & 34.60 & 24.91 & 34.23 & 62.6 & -- & -- \\
 & 32K & PI $8\times$ & CPT & 37.35 & 67.30 & 27.60 & 77.52 & 44.00 & \textbf{78.73} & 36.18 & 48.89 & 52.20 & 95.4 & 59.6 & 256M \\
 & 32K & PI $8\times$ & LongReD & 38.40 & 67.57 & 27.60 & \textbf{78.35} & 44.00 & \textbf{78.73} & 36.39 & 50.94 & 52.75 & 96.4 & 59.7 & 256M \\
 & 32K & PI $8\times$ & LinearARD & \textbf{39.54} & \textbf{67.82} & \textbf{29.31} & 76.82 & \textbf{44.20} & 78.40 & \textbf{39.82} & \textbf{51.11} & \textbf{53.38} & \textbf{97.6} & \textbf{65.1} & 256M \\
 & 32K & PI $8\times$ & CPT & 25.41 & 27.87 & 21.04 & 65.54 & 39.40 & 73.78 & 35.58 & 39.68 & 41.03 & 75.0 & 0.040 & 4.25M \\
 & 32K & PI $8\times$ & LongReD & 25.71 & 27.34 & 20.77 & 66.54 & 42.40 & 74.32 & 34.90 & 39.51 & 41.44 & 75.8 & 0.035 & 4.25M \\
 & 32K & PI $8\times$ & LinearARD & \textbf{34.50} & \textbf{63.40} & \textbf{27.70} & \textbf{76.54} & \textbf{43.90} & \textbf{77.30} & \textbf{35.98} & \textbf{48.10} & \textbf{50.93} & \textbf{93.1} & \textbf{63.0} & 4.25M \\
\midrule
\multirow{8}{*}{LLaMA3-8B} & 8K & -- & -- & 65.42 & 72.22 & 42.08 & 81.71 & 44.80 & 80.85 & 41.61 & 59.38 & 61.01 & 100.0 & -- & -- \\
 & 32K & PI $4\times$ & -- & 24.17 & 4.01 & 26.53 & 55.84 & 28.60 & 67.19 & 34.70 & 26.96 & 33.50 & 54.9 & -- & -- \\
 & 32K & PI $4\times$ & CPT & 60.74 & 69.38 & \textbf{39.87} & 80.24 & 43.60 & \textbf{80.36} & 39.00 & 57.42 & 58.83 & 96.4 & 81.3 & 256M \\
 & 32K & PI $4\times$ & LongReD & \textbf{61.26} & 68.98 & 38.83 & \textbf{80.98} & \textbf{45.20} & 80.00 & 39.82 & \textbf{57.59} & 59.08 & 96.9 & 67.9 & 256M \\
 & 32K & PI $4\times$ & LinearARD & 61.22 & \textbf{71.76} & 39.33 & 79.54 & 43.60 & 79.71 & \textbf{41.90} & 57.34 & \textbf{59.30} & \textbf{97.2} & \textbf{81.6} & 256M \\
 & 32K & PI $4\times$ & CPT & 35.20 & 48.67 & 29.92 & 71.01 & 36.20 & 73.94 & 34.70 & 43.69 & 46.66 & 76.5 & 0.043 & 4.25M \\
 & 32K & PI $4\times$ & LongReD & 41.18 & 44.74 & 31.79 & 73.82 & 37.00 & 74.05 & 34.85 & 46.04 & 47.93 & 78.6 & 0.044 & 4.25M \\
 & 32K & PI $4\times$ & LinearARD & \textbf{56.70} & \textbf{66.52} & \textbf{37.12} & \textbf{80.46} & \textbf{43.60} & \textbf{80.03} & \textbf{38.18} & \textbf{57.00} & \textbf{57.45} & \textbf{94.2} & \textbf{67.8} & 4.25M \\
\midrule
\multirow{8}{*}{Mistral-7B-v0.1} & 8K & -- & -- & 62.54 & 72.74 & 36.58 & 84.62 & 45.80 & 82.81 & 41.91 & 60.66 & 60.96 & 100.0 & -- & -- \\
 & 32K & PI $4\times$ & -- & 23.59 & 2.63 & 21.07 & 59.30 & 32.40 & 62.40 & 33.83 & 28.67 & 32.99 & 54.1 & -- & -- \\
 & 32K & PI $4\times$ & CPT & 53.53 & 67.94 & 33.70 & 83.43 & 43.60 & 79.71 & 39.82 & 56.57 & 57.29 & 94.0 & 55.3 & 256M \\
 & 32K & PI $4\times$ & LongReD & 58.82 & 69.13 & 35.24 & \textbf{84.50} & \textbf{45.80} & \textbf{81.61} & 40.99 & 58.98 & 59.38 & 97.4 & \textbf{62.3} & 256M \\
 & 32K & PI $4\times$ & LinearARD & \textbf{58.92} & \textbf{72.52} & \textbf{36.45} & 84.43 & 45.20 & 81.56 & \textbf{42.98} & \textbf{59.39} & \textbf{60.18} & \textbf{98.7} & 62.0 & 256M \\
 & 32K & PI $4\times$ & CPT & 24.83 & 46.32 & 23.12 & 64.62 & 37.00 & 70.62 & 35.94 & 37.46 & 42.49 & 69.7 & 0.010 & 4.25M \\
 & 32K & PI $4\times$ & LongReD & 24.20 & 28.78 & 23.45 & 63.94 & 37.40 & 69.26 & 35.37 & 39.08 & 40.19 & 65.9 & 0.010 & 4.25M \\
 & 32K & PI $4\times$ & LinearARD & \textbf{51.77} & \textbf{63.30} & \textbf{34.20} & \textbf{87.21} & \textbf{43.56} & \textbf{81.83} & \textbf{39.36} & \textbf{58.65} & \textbf{57.49} & \textbf{94.3} & \textbf{60.2} & 4.25M \\
\bottomrule
\end{tabular}}
\caption{Main results on short-text benchmarks and long-context robustness. Short-text accuracy in percent is reported across eight benchmarks and their mean, Avg. CW: context window; PE: RoPE scaling; Rec.: recovery relative to native; RULER: average over 8K, 16K, and 32K, with per-length breakdowns in App.~\ref{app:ruler_breakdown}. Token budgets use units of $2^{20}$ tokens across all stages, detailed in App.~\ref{app:token_accounting}. Bold marks the best trained result per backbone and budget group. LAMB.: LAMBADA; OBQA: OpenBookQA; ARC-C: ARC-Challenge.}
\label{tab:main_results}
\end{table*}

\paragraph{Training and optimization.}
For each backbone, the PI-scaled model acts as the student and the native-RoPE model as a frozen teacher.
The student is trained on sequences within the teacher's native context range using our final-layer $Q/Q, K/K$, and $V/V$ relation distillation objective (Eq.~\ref{eq:overall_loss}).
Unless otherwise stated, only the attention-module $Q/K/V$ projection weights are updated and all remaining parameters are kept fixed.
In addition to this distillation phase, a short continued-pretraining phase is run at the extended context length to further adapt the student to long sequences.
All models are optimized with AdamW, using a base learning rate of $2\times10^{-5}$, linear warmup followed by cosine decay, gradient clipping at 5.0, and gradient checkpointing for the student.

\paragraph{Baselines.}
Our comparisons include PI-scaled models without restoration, CPT under PI, and LongReD~\cite{longred2025}, which restores performance via hidden-state distillation.
For CPT and LongReD, the training budgets reported for this setting (256M tokens) are followed, while the restoration uses only a few million tokens and adds a lightweight 2M-token CPT phase.

\paragraph{Evaluation Benchmarks.}
Short-context performance is evaluated on MMLU~\cite{hendrycks2021mmlu} and seven standard short-text benchmarks, including LAMBADA~\cite{paperno2016lambada}, MathQA~\cite{amini2019mathqa}, BoolQ~\cite{clark2019boolq}, OpenBookQA~\cite{mihaylov2018openbookqa}, PIQA~\cite{bisk2020piqa}, SIQA-CA~\cite{sap2019socialiqa}, and ARC-Challenge~\cite{clark2018think}.
The mean over these eight scores is reported, denoted Avg.
Long-context robustness is measured with RULER~\cite{ruler2024}.
For all three backbones, evaluation is conducted at 8K, 16K, and 32K tokens and averaged across tasks and lengths (Appendix~\ref{app:ruler_breakdown}).
Additionally, we evaluate all three backbones across 13 tasks on the 32K-compatible subset of LongBench-E~\cite{bai2024longbench}; Appendix~\ref{app:longbench_protocol} details the tasks, length groupings, and complete breakdown scores.

\subsection{Main Results}

\paragraph{Short-context restoration.}
Table~\ref{tab:main_results} shows that PI scaling creates a large native-to-scaled mismatch: the short-text Avg.\ drops to 34.23/33.50/32.99 on LLaMA2-7B, LLaMA3-8B, and Mistral-7B-v0.1. With only 4.25M tokens, final-layer LinearARD recovers 93.1\%, 94.2\%, and 94.3\% of the native Avg., respectively, while budget-matched CPT and LongReD recover only 65.9--78.6\%. The gains are especially visible on geometry-sensitive tasks. For LLaMA2-7B, LAMBADA rises from 6.18 after PI scaling to 63.40 after LinearARD, and ARC-Challenge rises from 24.91 to 48.10. On LLaMA3-8B, LinearARD improves the 4.25M-token Avg.\ from 46.66/47.93 for CPT/LongReD to 57.45, narrowing the gap to the 61.01 native score. Similar recoveries appear on Mistral, where LinearARD reaches 57.49 Avg.\ while using roughly $60\times$ fewer tokens than the 256M-token baselines.

\paragraph{Long-context robustness and realistic evaluation.}
Under the constrained 4.25M-token budget, LinearARD achieves strong RULER performance across all backbones, scoring 63.0 on LLaMA2-7B, 67.8 on LLaMA3-8B, and 60.2 on Mistral-7B-v0.1. This matches or exceeds the 256M-token LongReD baseline, whereas budget-matched controls remain near zero, showing that standard language-modeling updates or hidden-state distillation fail to activate long-range retrieval at this small data scale. At the hardest 32K context, LinearARD consistently outperforms LongReD, for example lifting LLaMA2-7B from 41.8 to 47.4. Table~\ref{tab:longbench_cost_summary} broadens this assessment to realistic document comprehension on LongBench-E: 4.25M LinearARD substantially improves overall scores from 13.22 to 22.65 on LLaMA2-7B, from 12.80 to 21.18 on LLaMA3-8B, and from 4.78 to 17.17 on Mistral-7B-v0.1, matching 256M-token CPT baselines across both MHA and GQA architectures. While overall gains are substantial, minor length-dependent trade-offs occur, such as a slight dip from 22.82 to 22.43 in the 4--8K group for LLaMA3-8B. Detailed task and length breakdowns appear in Appendices~\ref{app:ruler_breakdown} and~\ref{app:longbench_protocol}.

\paragraph{Training efficiency and scaling transfer.}
As reported in Table~\ref{tab:longbench_cost_summary}, training GPU-hours under equal token budgets remain comparable to direct CPT, varying within $-30.5\%$ to $+21.7\%$ across all runs. Evaluated on a time-to-quality basis, however, LinearARD delivers dramatic compute savings: on LLaMA2-7B with Position Interpolation, 4.25M LinearARD reaches a LongBench-E score of 22.65 in 0.661 GPU-hours, whereas direct CPT requires 256M tokens and 31.593 GPU-hours to attain 22.94. This represents an empirical \textbf{47.8$\times$} reduction in training GPU-hours. Furthermore, LinearARD transfers effectively beyond Position Interpolation. Under fixed-$\theta$ NTK-aware scaling on LLaMA2-7B, 4.25M LinearARD lifts RULER from 32.85 to 64.84 and improves LongBench-E from 21.36 to 22.11 over budget-matched CPT, with consistent gains holding at 256M tokens. These results confirm that relation-level restoration generalizes robustly across distinct RoPE scaling formulations. Complete compute breakdowns and scaling comparisons appear in Appendices~\ref{app:kernel_throughput} and~\ref{app:scaling_transfer}.

\label{subsec:main_results}

\begin{table}[t]
\centering
\small
\renewcommand{\arraystretch}{1.12}
\begin{tabular*}{\linewidth}{@{\extracolsep{\fill}} ll cccc @{}}
\toprule
\multirow{2}{*}{\textbf{Backbone}} & \multirow{2}{*}{\textbf{Budget}} & \multicolumn{2}{c}{\textbf{LongBench-E} ($\uparrow$)} & \multicolumn{2}{c}{\textbf{GPU-Hours} ($\downarrow$)} \\
\cmidrule(lr){3-4} \cmidrule(lr){5-6}
 & & CPT & \textbf{LinearARD} & CPT & \textbf{LinearARD} \\
\midrule
\multirow{2}{*}{LLaMA2-7B} 
 & 4.25M & 13.220 & \textbf{22.650} & 0.545 & 0.661 \\
 & 256M  & 22.940 & \textbf{23.800} & 31.593 & 30.998 \\
\midrule
\multirow{2}{*}{LLaMA3-8B} 
 & 4.25M & 12.799 & \textbf{21.184} & 0.652 & 0.453 \\
 & 256M  & 21.392 & \textbf{22.799} & 23.924 & 27.120 \\
\midrule
\multirow{2}{*}{Mistral-7B-v0.1} 
 & 4.25M & 4.780  & \textbf{17.167} & 0.429 & 0.400 \\
 & 256M  & 20.663 & \textbf{21.796} & 19.591 & 23.847 \\
\bottomrule
\end{tabular*}
\caption{LongBench-E scores and training GPU-hours under PI scaling. Bold marks the best LongBench-E score per budget. Detailed breakdowns appear in Appendices~\ref{app:longbench_protocol} and~\ref{app:kernel_throughput}.}
\label{tab:longbench_cost_summary}
\end{table}

\subsection{Ablation Studies}
\label{subsec:ablation}

Table~\ref{tab:ablation_ard_components} compares distillation targets on LLaMA2-7B under PI $8\times$ without CPT. Full $Q/Q+K/K+V/V$ alignment achieves 56.05 Avg.\ and 93.64\% recovery, outperforming attention-map distillation at 53.84. Omitting $V/V$ drops Avg.\ to 54.27, confirming that value relations curb representation drift, while output-logit KL yields only a marginal gain to 57.00 at heavy memory cost. Layerwise diagnostics in Appendix~\ref{app:layerwise} show that restoration concentrates in upper layers: aligning only the final layer Last-1 captures 89.83\% recovery while running 10.9$\times$ faster than all-layer distillation, reducing step time from 252.61s to 23.22s at 32K context.

Tables~\ref{tab:ablation_cpt_short} and~\ref{tab:ablation_cpt_ruler} confirm a clear functional division between restoration and context activation. On LLaMA2-7B, ARD alone recovers short-context Avg.\ to 51.7, close to the 51.9 reached with CPT, but achieves only 36.8 on RULER. Adding the 2M-token CPT stage then surges RULER to 63.2. Consistent trends across GQA backbones show that restoring internal attention geometry first enables context extension with minimal long-sequence tokens.

\section{Conclusion}
\label{sec:conclusion}

In this paper, we presented LinearARD. By distilling $Q/Q$, $K/K$, and $V/V$ relational distributions from a native-RoPE teacher, LinearARD directly rectifies the fine-grained positional distortions in the attention mechanism. Our analysis of training dynamics reveals a strong causal link between internal attention drift and output distribution shifts: aligning the internal relational structure implicitly restores the model's output proficiency, rendering explicit logit-level supervision redundant. To overcome the quadratic memory barrier, an IO-aware, linear-activation-memory kernel is introduced to enable exact dense distillation on ultra-long sequences. Empirical evaluations on LLaMA2-7B, LLaMA3-8B, and Mistral-7B-v0.1 show that LinearARD restores short-context performance with about $60\times$ fewer tokens than standard continued pre-training and provides strong long-context robustness. This work underscores the importance of maintaining structural integrity within the attention mechanism for efficient and effective context window extension.

% 正文结束后、参考文献之前
\begin{ack}
This work was supported by the National Natural Science Foundation
of China under Grant 62301559, and the Beijing High-Creation
Young Talent Support Project.
\end{ack}

% Intentionally left blank for now.

\bibliography{references.bib}
\bibliographystyle{plainnat}

%%%%%%%%%%%%%%%%%%%%%%%%%%%%%%%%%%%%%%%%%%%%%%%%%%%%%%%%%%%%%%%%%%%%%%%%%%%%%%%
%%%%%%%%%%%%%%%%%%%%%%%%%%%%%%%%%%%%%%%%%%%%%%%%%%%%%%%%%%%%%%%%%%%%%%%%%%%%%%%
% APPENDIX
%%%%%%%%%%%%%%%%%%%%%%%%%%%%%%%%%%%%%%%%%%%%%%%%%%%%%%%%%%%%%%%%%%%%%%%%%%%%%%%
%%%%%%%%%%%%%%%%%%%%%%%%%%%%%%%%%%%%%%%%%%%%%%%%%%%%%%%%%%%%%%%%%%%%%%%%%%%%%%%
\newpage
\appendix

\phantomsection
\section*{Appendix Roadmap}
\addcontentsline{toc}{section}{Appendix Roadmap}
This supplement is organized to make it easy to locate experimental details, theory derivations, and diagnostic evidence:
\begin{itemize}
    \item \hyperref[app:ablation_details]{\textbf{Appendix A:} Detailed ablation results for ARD components and the CPT stage.}
    \item \hyperref[app:token_accounting]{\textbf{Appendix B:} Training configuration, exact token accounting, and evaluation protocol.}
    \item \hyperref[app:kernel_throughput]{\textbf{Appendix C:} Kernel benchmarks and complete training GPU-hour audit.}
    \item \hyperref[app:aux_objectives]{\textbf{Appendix D:} Auxiliary objectives used in ablation studies.}
    \item \hyperref[app:layerwise]{\textbf{Appendix E:} Layerwise distillation results and interpretation.}
    \item \hyperref[app:extreme_scaling]{\textbf{Appendix F:} Extreme-scaling short-context stress tests.}
    \item \hyperref[app:ruler_breakdown]{\textbf{Appendix G:} RULER per-length and per-task breakdowns.}
    \item \hyperref[app:proof_gradient]{\textbf{Appendix H:} Proof of Proposition~\ref*{prop:grad_sensitivity}.}
    \item \hyperref[app:complexity]{\textbf{Appendix I:} Proof of Theorem~\ref*{thm:memory}.}
    \item \hyperref[app:exactness]{\textbf{Appendix J:} Proof of Proposition~\ref*{prop:exactness}.}
    \item \hyperref[app:backward_kernel]{\textbf{Appendix K:} KL distillation backward-kernel pseudocode.}
    \item \hyperref[app:supplementary_analysis]{\textbf{Appendix L:} Supplementary analyses and numerical verification.}
    \item \hyperref[app:longbench_protocol]{\textbf{Appendix M:} LongBench-E protocol and detailed results.}
    \item \hyperref[app:scaling_transfer]{\textbf{Appendix N:} PI and fixed-theta NTK configuration and results.}
    \item \hyperref[app:broader_impacts]{\textbf{Appendix O:} Broader impacts.}
\end{itemize}

\section{Detailed Ablation Results}
\label{app:ablation_details}

This section provides comprehensive breakdowns of the ablation studies summarized in Sec.~\ref{subsec:ablation}. To isolate the effects of individual algorithmic components without confounding variables, we conduct controlled ablations using fixed random seeds and identical training checkpoints.

\begin{table}[H]
\centering
\footnotesize
\renewcommand{\arraystretch}{1.06}
\setlength{\tabcolsep}{3.8pt}
\begin{tabular*}{\linewidth}{@{\extracolsep{\fill}} l l r r r r r r}
\toprule
PE & Variant & MMLU & LAMB. & BoolQ & OBQA & Avg. & Rec. \\
\midrule
-- & Native & 45.99 & 71.18 & 78.23 & 44.00 & 59.85 & 100.00 \\
\midrule
PI $8\times$ & w/o restore & 24.65 & 6.18 & 58.29 & 36.20 & 31.33 & 52.35 \\
 & ARD (attn. variant) & 31.63 & 65.17 & 76.76 & 41.80 & 53.84 & 89.95 \\
 & \textbf{ARD (QKV rel.)} & \textbf{36.81} & 64.72 & 79.05 & \textbf{43.60} & \textbf{56.05} & \textbf{93.64} \\
 & ARD (+logit KL) & 36.38 & \textbf{68.85} & \textbf{79.17} & \textbf{43.60} & 57.00 & 95.24 \\
 & ARD (full FT QKV) & 30.42 & 58.66 & 75.38 & 43.60 & 52.02 & 86.91 \\
 & ARD (full FT LLM) & 31.03 & 60.02 & 76.36 & 42.60 & 52.50 & 87.72 \\
 & ARD (w/o V/V) & 33.44 & 61.15 & 79.11 & 43.40 & 54.27 & 90.68 \\
 & ARD (LoRA QKVO+MLP) & 34.13 & 63.30 & 77.74 & 43.20 & 54.59 & 91.21 \\
\bottomrule
\end{tabular*}
\caption{ARD component ablation on LLaMA2-7B under PI $8\times$ without CPT across all layers. Avg.\ is the mean of MMLU, LAMBADA, BoolQ, and OpenBookQA. Rec.\ is the recovery ratio relative to the native model.}
\label{tab:ablation_ard_components}
\end{table}

\begin{table}[H]
\centering
\footnotesize
\renewcommand{\arraystretch}{1.06}
\setlength{\tabcolsep}{3.2pt}
\begin{tabular*}{\linewidth}{@{\extracolsep{\fill}} l l r r r r r r r r r}
\toprule
\multicolumn{2}{c}{} & \multicolumn{9}{c}{Short-text accuracy (\%)} \\
\cmidrule(lr){3-11}
PE & Method & MMLU & LAMB. & MathQA & BoolQ & OBQA & PIQA & SIQA & ARC-C & Avg. \\
\midrule
-- & Native & 46.0 & 71.2 & 29.9 & 78.2 & 44.0 & 78.7 & 40.3 & 49.2 & 54.7 \\
PI $8\times$ & ARD (no CPT) & 36.3 & 63.2 & 27.2 & 79.4 & 43.8 & 78.1 & 36.3 & 49.4 & 51.7 \\
PI $8\times$ & ARD + CPT & 36.8 & 64.7 & 27.9 & 77.5 & 43.8 & 78.6 & 36.0 & 49.6 & 51.9 \\
\bottomrule
\end{tabular*}
\caption{Short-context ablation of the two training stages on LLaMA2-7B under PI $8\times$. ARD denotes relation distillation on PG19, and CPT denotes subsequent lightweight adaptation on SlimPajama.}
\label{tab:ablation_cpt_short}
\end{table}

\paragraph{Stage attribution across backbones.}
To evaluate how Stage~1 (ARD) and Stage~2 (CPT) interact across diverse architectures, Table~\ref{tab:stage_attribution} provides a comprehensive stage-wise evaluation across LLaMA2-7B, LLaMA3-8B, and Mistral-7B-v0.1 under PI scaling. We compare Stage~1 alone (ARD only), Stage~2 alone (CPT only), and the complete two-stage pipeline (ARD + CPT). Across both standard MHA and Grouped-Query Attention (GQA) architectures, ARD alone recovers the overwhelming majority of short-context capabilities, while the lightweight CPT stage activates full long-range multi-needle retrieval across extended contexts. Conversely, applying CPT alone under this severely constrained token budget yields near-zero retrieval on RULER, demonstrating that internal relation restoration is the critical driver of sample-efficient context extension.

\begin{table}[t]
\centering
\footnotesize
\renewcommand{\arraystretch}{1.12}
\setlength{\tabcolsep}{0pt}
\begin{tabular*}{\linewidth}{@{\extracolsep{\fill}} l l r r r r r r @{}}
\toprule
\multirow{2}{*}{\textbf{Backbone}} & \multirow{2}{*}{\textbf{Condition}} & \multirow{2}{*}{\textbf{Tokens}} & \multirow{2}{*}{\textbf{Short Avg.}} & \multicolumn{4}{c}{\textbf{RULER Accuracy (\%)}} \\
\cmidrule(lr){5-8}
 & & & & \textbf{8K} & \textbf{16K} & \textbf{32K} & \textbf{Avg.} \\
\midrule
\multirow{3}{*}{LLaMA2-7B} 
 & CPT only    & 4.25M & 41.03 & 0.06 & 0.02 & 0.04 & 0.04 \\
 & ARD only    & 2.25M & 51.70 & 46.50 & 42.50 & 21.40 & 36.80 \\
 & ARD + CPT   & 4.25M & 51.90 & 74.90 & 67.40 & 47.30 & 63.20 \\
\midrule
\multirow{3}{*}{LLaMA3-8B} 
 & CPT only    & 4.25M & 46.66 & 0.05 & 0.04 & 0.04 & 0.04 \\
 & ARD only    & 2.00M & 56.58 & 62.08 & 61.29 & 46.91 & 56.76 \\
 & ARD + CPT   & 4.25M & 57.45 & 79.70 & 71.70 & 51.90 & 67.80 \\
\midrule
\multirow{3}{*}{Mistral-7B-v0.1} 
 & CPT only    & 4.25M & 42.49 & 0.01 & 0.01 & 0.01 & 0.01 \\
 & ARD only    & 2.00M & 57.28 & 52.81 & 36.74 & 7.99 & 32.51 \\
 & ARD + CPT   & 4.25M & 57.49 & 79.40 & 64.10 & 37.00 & 60.20 \\
\bottomrule
\end{tabular*}
\caption{Comprehensive stage attribution across model backbones under PI scaling. We report the 8-benchmark short-text accuracy (Short Avg.) and per-length RULER performance across 8K, 16K, and 32K context windows. Stage~1 (ARD) drives short-context restoration, while Stage~2 (CPT) unlocks long-range context retrieval.}
\label{tab:stage_attribution}%
\label{tab:ablation_cpt_ruler}% 双label绑定，保证原有引用均不失效
\end{table}

\section{Training Efficiency and Token Accounting}
\label{app:token_accounting}

Tokens are reported as the number of non-padding tokens consumed during training. The training pipeline streams text data and packs it into fixed-length blocks at the stage context length, so each sample contains entirely valid tokens. The token budget is computed deterministically as:
\begin{equation}
\label{eq:token_budget}
N_{\mathrm{tok}} = \sum_{s} L_s \cdot B_{\mathrm{global},s} \cdot A_s \cdot U_s,
\end{equation}
where $L_s$ is the stage context length, $B_{\mathrm{global},s}$ is the global batch size ($B_{\mathrm{global},s} = G_s \cdot B_{\mathrm{device},s}$ for distributed ARD across $G_s$ GPUs, and set directly for CPT), $A_s$ is the gradient accumulation factor, and $U_s$ is the number of optimizer steps. 

Budget labels use binary units ($1\text{M} = 2^{20} = 1{,}048{,}576$ tokens). The 4.25M budget corresponds to exactly $4{,}456{,}448$ tokens, which decomposes into $2{,}359{,}296$ tokens for Stage~1 (ARD) and $2{,}097{,}152$ tokens for Stage~2 (CPT), strictly matching the budget of direct CPT. The 256M baseline budget corresponds to $268{,}435{,}456$ tokens, yielding a $\sim 60.2\times$ token reduction for LinearARD. Table~\ref{tab:exact_training_tokens} provides the complete training token accounting across models and budgets.

\begin{table}[H]
\centering
\footnotesize
\renewcommand{\arraystretch}{1.08}
\setlength{\tabcolsep}{0pt}
\begin{tabular*}{\linewidth}{@{\extracolsep{\fill}} l l r r l r r r @{}}
\toprule
\textbf{Model} & \textbf{Stage} & \textbf{Context} & \textbf{GPUs} & \textbf{Batch Size} & \textbf{Accum.} & \textbf{Steps} & \textbf{Exact Tokens} \\
\midrule
\multirow{3}{*}{LLaMA2-7B (4.25M)} 
 & Stage 1 (ARD)  & 4,096  & 2 & 2 / device & 1 & 144 & 2,359,296 \\
 & Stage 2 (CPT)  & 32,768 & 4 & 4 global   & 1 & 16  & 2,097,152 \\
 & Direct CPT     & 32,768 & 8 & 17 global  & 1 & 8   & 4,456,448 \\
\midrule
LLaMA3 / Mistral (4.25M)
 & Stage 1 (ARD)  & 8,192  & 2 & 1 / device & 1 & 144 & 2,359,296 \\
\midrule
\multirow{2}{*}{All Models (256M)} 
 & Stage 1 (ARD)  & 4K / 8K & 2 & 8 / device & 8 / 4 & 256 & 134,217,728 \\
 & Stage 2 (CPT)  & 32,768  & 4 & 4 global   & 16    & 64  & 134,217,728 \\
\bottomrule
\end{tabular*}
\caption{Exact token accounting for LinearARD and CPT baselines across budgets. Stage~1 uses PG19 and Stage~2 uses SlimPajama-627B. The combined two-stage LinearARD tokens strictly match the CPT baseline.}
\label{tab:exact_training_tokens}
\end{table}

\paragraph{Training and Evaluation Settings.}
Stage~1 (ARD) trains QKV LoRA ($r=1024, \alpha=2048$) using AdamW (base LR $2\times 10^{-5}$, cosine decay, warmup of 4 steps, clipping at 5.0) with $\lambda_q=\lambda_k=\lambda_v=1.0$ and temperature $1.0$. Stage~2 merges the adapter and runs 32K context continuation on SlimPajama using FlashAttention 2 and zigzag ring attention. Downstream short-context evaluations use lm-evaluation-harness (MMLU 5-shot, MathQA 4-shot, BoolQ 3-shot, ARC-Challenge 25-shot, others 0-shot), and long-context evaluations use RULER 0-shot at 8K, 16K, and 32K across all 12 tasks with truncation disabled.

\section{Kernel and Training Throughput}
\label{app:kernel_throughput}

This section provides single-operator kernel benchmarks, layer-scope training throughput, and the complete training GPU-hour audit supporting Section~\ref{subsec:main_results} and Table~\ref{tab:longbench_cost_summary}. All latency and throughput benchmarks are evaluated on NVIDIA RTX PRO 6000 GPUs (96GB memory) paired with Intel Xeon Platinum 8470Q processors, averaged over 20 training steps.

\paragraph{Operator Latency and Memory Feasibility.}
Table~\ref{tab:kernel_operator_latency} compares the single-operator execution latency between a standard dense implementation and our IO-aware tiled kernel. While the naive implementation exhausts GPU memory at sequence lengths of 16K and beyond, our kernel scales to 32K context with linear activation memory. Furthermore, by avoiding the materialization and repeated HBM transfer of $n \times n$ probability matrices, the fused kernel achieves a $1.60\times$--$1.90\times$ speedup across all sequence lengths where dense calculation fits in memory.

\begin{table}[tbp]
\caption{Single-operator latency for the naive dense implementation and the IO-aware kernel. The naive implementation runs out of memory (OOM) at 16K and above.}
\label{tab:kernel_operator_latency}
\centering
\small
\renewcommand{\arraystretch}{1.08}
\setlength{\tabcolsep}{6pt}
\begin{tabular}{l c c c c c}
\toprule
\textbf{Seq. Len.} & \textbf{Naive (ms)} & \textbf{Kernel (ms)} & \textbf{Speedup} & \textbf{Naive Status} & \textbf{Kernel Status} \\
\midrule
1K  & 5.30   & 3.32    & 1.60$\times$ & OK  & OK \\
2K  & 20.74  & 11.83   & 1.75$\times$ & OK  & OK \\
4K  & 81.25  & 43.82   & 1.85$\times$ & OK  & OK \\
8K  & 326.93 & 172.56  & 1.90$\times$ & OK  & OK \\
16K & OOM    & 683.54  & --           & OOM & OK \\
32K & OOM    & 2549.54 & --           & OOM & OK \\
\bottomrule
\end{tabular}
\end{table}

\paragraph{Layer-Scope Training Throughput.}
Table~\ref{tab:end_to_end_time_throughput} reports end-to-end training step times and throughput across context lengths and layer scopes on LLaMA2-7B. Distilling dense relations across all 32 layers incurs heavy computational overhead, with the relation kernel accounting for up to 94.5\% of training step time at 32K. Aligning only the final layer (Last-1) reduces step latency from 252.61s to 23.22s at 32K context—a \textbf{10.9$\times$} throughput improvement—making final-layer distillation the highly efficient default.

\begin{table}[tbp]
\caption{End-to-end ARD training step time and throughput under different layer scopes. All-32 distills all layers, while Last-$k$ distills only the final $k$ layers.}
\label{tab:end_to_end_time_throughput}
\centering
\scriptsize
\renewcommand{\arraystretch}{1.08}
\resizebox{\textwidth}{!}{%
\begin{tabular}{l r r r r r r r r r r r r}
\toprule
\multirow{2}{*}{\textbf{Context}} & \multicolumn{6}{c}{\textbf{Step Time (s)}} & \multicolumn{6}{c}{\textbf{Throughput (tokens/s)}} \\
\cmidrule(lr){2-7} \cmidrule(lr){8-13}
 & All-32 & Last-16 & Last-8 & Last-4 & Last-2 & Last-1
 & All-32 & Last-16 & Last-8 & Last-4 & Last-2 & Last-1 \\
\midrule
1K  & 0.77   & 0.62   & 0.54  & 0.50  & 0.49  & 0.48  & 1327.2 & 1655.6 & 1882.2 & 2034.6 & 2070.9 & 2144.0 \\
2K  & 1.84   & 1.30   & 1.02  & 0.88  & 0.81  & 0.78  & 1113.8 & 1581.4 & 2005.7 & 2320.1 & 2519.2 & 2629.0 \\
4K  & 5.50   & 3.47   & 2.45  & 1.94  & 1.68  & 1.56  & 744.7  & 1181.9 & 1672.5 & 2108.4 & 2431.5 & 2628.1 \\
8K  & 19.08  & 11.09  & 7.11  & 5.11  & 4.11  & 3.61  & 429.4  & 738.4  & 1152.9 & 1602.2 & 1991.3 & 2267.1 \\
16K & 76.32  & 38.35  & 22.53 & 14.64 & 10.69 & 8.70  & 214.7  & 427.2  & 727.2  & 1119.3 & 1532.6 & 1882.9 \\
32K & 252.61 & 143.53 & 78.03 & 46.70 & 31.04 & 23.22 & 129.7  & 228.3  & 419.9  & 701.7  & 1055.7 & 1411.4 \\
\bottomrule
\end{tabular}}
\end{table}

\paragraph{Training GPU-Hour Audit and Compute Efficiency.}
For a training pipeline comprising stages $s$ with $G_s$ GPUs, $U_s$ optimizer updates, and mean step latency $\tau_s$ seconds, total device compute is measured as:
\begin{equation}
\text{GPU-hours} = \frac{1}{3600}\sum_s G_s \cdot U_s \cdot \tau_s.
\end{equation}
Table~\ref{tab:training_gpu_hours} reports the exact compute expenditure across models and budgets under PI scaling, corresponding to Table~\ref{tab:longbench_cost_summary} in the main text. On LLaMA2-7B, the 4.25M LinearARD pipeline consumes 0.661 GPU-hours (Stage~1: 144 steps on 2 GPUs at 3.340s/step; Stage~2: 16 steps on 4 GPUs at 22.163s/step), whereas 256M direct CPT consumes 31.593 GPU-hours (128 steps on 4 GPUs at 222.139s/step). 

As shown in Table~\ref{tab:training_gpu_hours}, when evaluated under equal token budgets, LinearARD incurs compute requirements comparable to direct CPT (ranging from $-30.5\%$ to $+21.7\%$). When evaluated on a time-to-quality basis, however, 4.25M LinearARD achieves a LongBench-E score of 22.65 in 0.661 GPU-hours, matching the 22.94 score achieved by 256M CPT in 31.593 GPU-hours. This represents an empirical \textbf{47.8$\times$} compute reduction ($31.593 \,/\, 0.661 \approx 47.8$).

\begin{table}[tbp]
\caption{Complete training GPU-hour audit across model backbones and budgets under PI scaling. Costs correspond to the LongBench-E evaluations in Table~\ref{tab:longbench_cost_summary}.}
\label{tab:training_gpu_hours}
\centering
\small
\renewcommand{\arraystretch}{1.10}
\setlength{\tabcolsep}{0pt}
\begin{tabular*}{\linewidth}{@{\extracolsep{\fill}} l l r r r @{}}
\toprule
\textbf{Backbone} & \textbf{Budget} & \textbf{LinearARD (GPU-h)} & \textbf{Direct CPT (GPU-h)} & \textbf{Cost Difference} \\
\midrule
\multirow{2}{*}{LLaMA2-7B} 
 & 4.25M & 0.661  & 0.545  & +21.3\% \\
 & 256M  & 30.998 & 31.593 & --1.9\% \\
\midrule
\multirow{2}{*}{LLaMA3-8B} 
 & 4.25M & 0.453  & 0.652  & --30.5\% \\
 & 256M  & 27.120 & 23.924 & +13.4\% \\
\midrule
\multirow{2}{*}{Mistral-7B-v0.1} 
 & 4.25M & 0.400  & 0.429  & --6.8\% \\
 & 256M  & 23.847 & 19.591 & +21.7\% \\
\bottomrule
\end{tabular*}
\end{table}

\section{Auxiliary Objectives (Ablations)}
\label{app:aux_objectives}

To evaluate alternative distillation targets, two auxiliary objectives are defined for ablation studies.

\paragraph{Attention-Map KL.}
For a given attention head, let $\mathbf{q}_{m,i}\in\mathbb{R}^{d}$ and $\mathbf{k}_{m,j}\in\mathbb{R}^{d}$ denote the query and key representations at positions $i$ and $j$ from model $m \in \{t, s\}$. The scaled dot-product logits and row-wise attention probabilities are:
\begin{subequations}\label{eq:attn_def}
\begin{align}
    \mathbf{S}_m(i, j) &= \frac{\mathbf{q}_{m,i}\,\mathbf{k}_{m,j}^{\top}}{\sqrt{d}} + \mathbf{M}(i, j), \label{eq:attn_logits}\\
    \mathbf{A}_m(i, j) &= \frac{\exp(\mathbf{S}_m(i, j))}{\sum_{k=1}^n \exp(\mathbf{S}_m(i, k))}. \label{eq:attn_probs}
\end{align}
\end{subequations}
The attention-map distillation loss aligns these distributions via forward KL: $\mathcal{L}_{\text{attn}} = \mathcal{L}_{\mathrm{KL}}(\mathbf{A}_t, \mathbf{A}_s)$. In Table~\ref{tab:ablation_ard_components}, the ``ARD (attention)'' variant replaces the $Q/Q$ and $K/K$ terms in Eq.~\eqref{eq:overall_loss} with $\mathcal{L}_{\text{attn}}$, while retaining the $V/V$ relation constraint.

\paragraph{Output-Logit KL.}
Distillation on vocabulary logits $\mathbf{z}_m \in \mathbb{R}^{V}$ applies temperature-scaled KL divergence:
\begin{equation}
\label{eq:logits_kl}
    \mathcal{L}_{\text{logits}} = D_{\mathrm{KL}}\!\left(\mathrm{softmax}(\mathbf{z}_t/T) \parallel \mathrm{softmax}(\mathbf{z}_s/T)\right).
\end{equation}
In Table~\ref{tab:ablation_ard_components}, ``+logit KL'' incorporates $\mathcal{L}_{\text{logits}}$ alongside the default relation objective in Eq.~\eqref{eq:overall_loss}, while ``w/o V/V'' removes the value-side relation term.

\section{Layerwise Distillation}
\label{app:layerwise}

Figure~\ref{fig:layerwise_selection} evaluates single-layer ARD across individual Transformer layers on LLaMA2-7B under PI $8\times$ without CPT. The earliest layers (1--4) provide minimal recovery signal, whereas restoration gains surge through layers 5--8 and plateau in the final layers. Aligning only the final layer ($\ell^\star = L$) captures 89.83\% recovery in this setting (compared to 93.64\% for all 32 layers). As documented in Table~\ref{tab:end_to_end_time_throughput}, restricting distillation to the final layer reduces step time from 252.61s to 23.22s at 32K context—yielding a \textbf{10.9$\times$} training speedup while retaining the vast majority of the restoration benefit.

\begin{table}[H]
\centering
\scriptsize
\renewcommand{\arraystretch}{1.10}
\setlength{\tabcolsep}{3.5pt}
\resizebox{\textwidth}{!}{%
\begin{tabular}{l c l r r r r r r r r r r}
\toprule
\textbf{PE} & \textbf{Context} & \textbf{Method} & \textbf{MMLU} & \textbf{LAMB.} & \textbf{MathQA} & \textbf{BoolQ} & \textbf{OBQA} & \textbf{PIQA} & \textbf{SIQA} & \textbf{ARC-C} & \textbf{Avg.} & \textbf{Rec. (\%)} \\
\midrule
-- & 4K & Native & 45.99 & 71.18 & 29.92 & 78.23 & 44.00 & 78.73 & 40.28 & 49.23 & 54.69 & 100.00 \\
\midrule
PI $8\times$ & 32K & w/o restore & 24.65 & 6.18 & 19.73 & 58.29 & 36.20 & 69.31 & 34.60 & 24.91 & 34.23 & 62.60 \\
PI $8\times$ & 32K & ARD alone & 36.81 & 64.72 & 27.94 & 77.54 & 43.80 & 78.62 & 35.98 & 49.57 & 51.87 & 94.85 \\
\midrule
PI $16\times$ & 64K & w/o restore & 23.22 & 0.66 & 19.53 & 44.04 & 32.60 & 63.71 & 34.54 & 24.83 & 30.39 & 55.57 \\
PI $16\times$ & 64K & ARD alone & 32.72 & 59.89 & 26.26 & 77.55 & 42.20 & 77.53 & 36.30 & 49.15 & 50.20 & 91.79 \\
\midrule
PI $32\times$ & 128K & w/o restore & 22.98 & 0.35 & 18.99 & 38.44 & 32.40 & 61.86 & 34.08 & 24.32 & 29.18 & 53.35 \\
PI $32\times$ & 128K & ARD alone & 30.05 & 55.75 & 25.76 & 74.34 & 43.00 & 77.42 & 36.15 & 46.84 & 48.66 & 88.98 \\
\bottomrule
\end{tabular}}
\caption{Short-context evaluation under extreme PI scaling factors on LLaMA2-7B without CPT. ARD alone restores over 88\%--91\% of native accuracy across extreme context extensions.}
\label{tab:extreme_scaling_short}
\end{table}

\begin{table}[H]
\centering
\small
\renewcommand{\arraystretch}{1.12}
\setlength{\tabcolsep}{0pt}
\begin{tabular*}{\linewidth}{@{\extracolsep{\fill}} l l r r r r @{}}
\toprule
\textbf{Backbone} & \textbf{Method (Budget)} & \textbf{8K} & \textbf{16K} & \textbf{32K} & \textbf{Avg. All} \\
\midrule
\multirow{2}{*}{LLaMA2-7B} 
 & CPT (4.25M)     & 0.0625 & 0.0192 & 0.0383 & 0.040 \\
 & LongReD (4.25M) & 0.0458 & 0.0292 & 0.0300 & 0.035 \\
\midrule
\multirow{2}{*}{LLaMA3-8B} 
 & CPT (4.25M)     & 0.0475 & 0.0442 & 0.0375 & 0.043 \\
 & LongReD (4.25M) & 0.0408 & 0.0400 & 0.0508 & 0.044 \\
\midrule
\multirow{2}{*}{Mistral-7B-v0.1} 
 & CPT (4.25M)     & 0.0100 & 0.0108 & 0.0100 & 0.010 \\
 & LongReD (4.25M) & 0.0150 & 0.0075 & 0.0075 & 0.010 \\
\bottomrule
\end{tabular*}
\caption{Per-length RULER scores for budget-matched CPT and LongReD controls (4.25M tokens) under PI scaling. The exact length-wise scores explain the collapsed retrieval behavior, and their aggregate mean (Avg. All) strictly corresponds to the main results in Table~\ref{tab:main_results}.}
\label{tab:budget_matched_ruler_lengths}
\end{table}

\begin{figure}[H]
\centering
\maybeincludegraphics[width=0.93\linewidth]{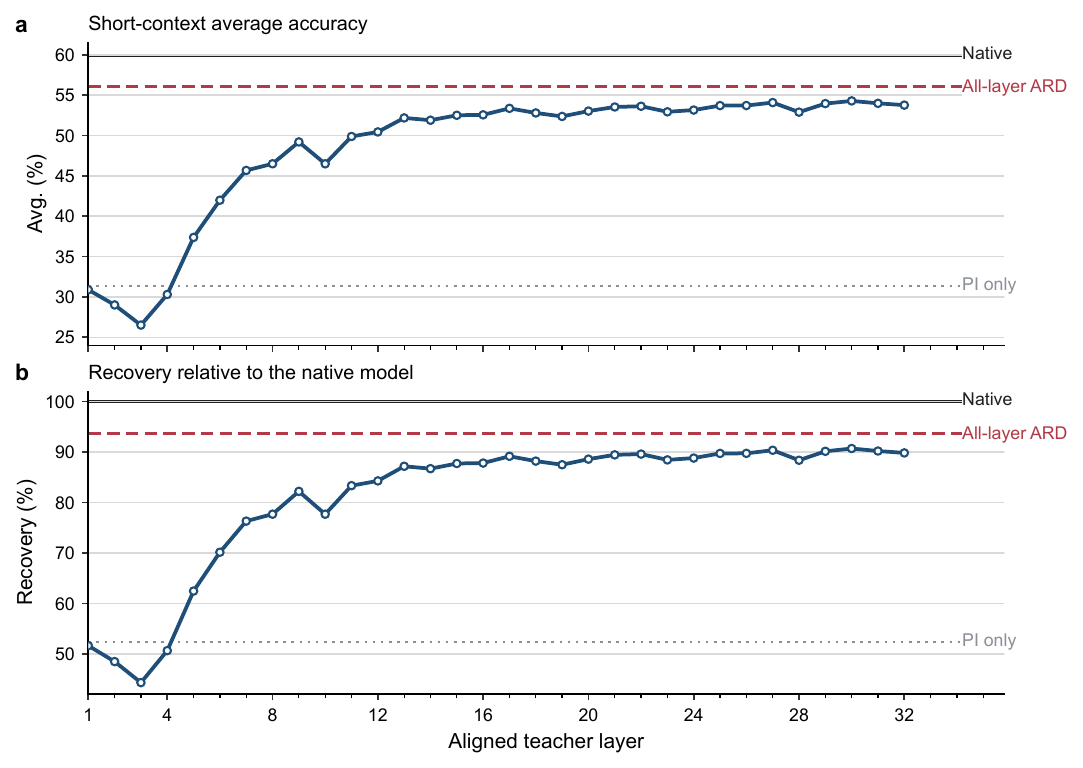}
\caption{Layerwise restoration contribution of ARD on LLaMA2-7B under PI $8\times$ without CPT. Deeper layers capture the dominant relational recovery signal, enabling highly efficient single-layer distillation.}
\label{fig:layerwise_selection}
\end{figure}

\section{Extreme Scaling Short-Context Stress Test}
\label{app:extreme_scaling}

To assess how relational distillation performs under extreme positional scaling factors, Table~\ref{tab:extreme_scaling_short} evaluates short-context accuracy on LLaMA2-7B extended to 64K ($16\times$) and 128K ($32\times$) contexts using Stage~1 ARD alone (288 steps, 4.5M distillation tokens, without CPT). Even under extreme frequency compression, ARD restores short-context recovery from 55.57\% to 91.79\% at 64K and from 53.35\% to 88.98\% at 128K, demonstrating the structural stability of relation-level restoration across substantial scaling regimes.

\section{RULER Breakdown}
\label{app:ruler_breakdown}

Table~\ref{tab:budget_matched_ruler_lengths} summarizes the per-length RULER scores for budget-matched CPT and LongReD controls. Under the restricted 4.25M-token budget, both baselines remain near zero across all evaluation lengths, confirming that language modeling loss or hidden-state distillation alone is insufficient to activate retrieval under severe data constraints. 

Tables~\ref{tab:ruler_llama2},~\ref{tab:ruler_llama3}, and~\ref{tab:ruler_mistral} report detailed subtask breakdowns across 8K, 16K, and 32K context windows for LLaMA2-7B, LLaMA3-8B, and Mistral-7B-v0.1. On LLaMA2-7B, 4.25M LinearARD achieves strong 32K retrieval (47.4 vs.\ 41.8 for 256M LongReD), with major gains on multi-value retrieval (NIAH-MV: 59.7 vs.\ 23.0). On LLaMA3-8B, LinearARD attains a 32K average of 51.9, outperforming 256M LongReD (50.0).

\begin{table*}[t]
\centering
\scriptsize
\renewcommand{\arraystretch}{1.10}
\setlength{\tabcolsep}{2.5pt}
\begin{tabular*}{\textwidth}{@{\extracolsep{\fill}} l ccc ccc ccc ccc @{}}
\toprule
\multirow{2}{*}{\textbf{Task}} & \multicolumn{3}{c}{\textbf{CPT (256M)}} & \multicolumn{3}{c}{\textbf{LongReD (256M)}} & \multicolumn{3}{c}{\textbf{LinearARD (4.25M)}} & \multicolumn{3}{c}{\textbf{LinearARD (no CPT)}} \\
\cmidrule(lr){2-4} \cmidrule(lr){5-7} \cmidrule(lr){8-10} \cmidrule(lr){11-13}
 & \textbf{8K} & \textbf{16K} & \textbf{32K} & \textbf{8K} & \textbf{16K} & \textbf{32K} & \textbf{8K} & \textbf{16K} & \textbf{32K} & \textbf{8K} & \textbf{16K} & \textbf{32K} \\
\midrule
NIAH-S1   & 95.0 & 91.0 & 82.0 & 95.0 & 94.0 & 88.0 & 100.0 & 100.0 & 100.0 & 97.0 & 90.0 & 76.0 \\
NIAH-S2   & 95.0 & 96.0 & 78.0 & 95.0 & 97.0 & 86.0 & 100.0 & 100.0 & 96.0  & 80.0 & 74.0 & 30.0 \\
NIAH-S3   & 91.0 & 84.0 & 55.0 & 95.0 & 91.0 & 60.0 & 88.0  & 88.0  & 43.0  & 45.0 & 52.0 & 15.0 \\
NIAH-K1   & 93.0 & 90.0 & 62.0 & 95.0 & 94.0 & 63.0 & 96.0  & 96.0  & 76.0  & 66.0 & 58.0 & 25.0 \\
NIAH-K2   & 76.0 & 68.0 & 23.0 & 76.0 & 60.0 & 20.0 & 74.0  & 70.0  & 29.0  & 47.0 & 43.0 & 13.0 \\
NIAH-K3   & 22.0 & 10.0 & 0.0  & 26.0 & 11.0 & 0.0  & 27.0  & 0.0   & 0.0   & 1.0  & 0.0  & 1.0  \\
NIAH-MQ   & 90.8 & 87.2 & 52.0 & 91.0 & 89.0 & 57.0 & 92.5  & 86.9  & 49.3  & 65.0 & 54.3 & 28.0 \\
NIAH-MV   & 82.8 & 60.0 & 32.0 & 83.0 & 57.0 & 23.0 & 96.6  & 88.8  & 59.7  & 47.3 & 28.0 & 15.8 \\
VT        & 84.8 & 37.6 & 30.4 & 86.0 & 31.0 & 13.0 & 59.5  & 43.0  & 39.1  & 10.8 & 11.2 & 0.0  \\
CWE       & 36.0 & 34.0 & 41.4 & 36.0 & 25.0 & 36.0 & 43.7  & 32.3  & 16.6  & 52.2 & 45.1 & 27.9 \\
FWE       & 67.3 & 61.0 & 29.0 & 68.0 & 60.0 & 39.0 & 63.5  & 66.1  & 43.9  & 15.7 & 21.7 & 7.7  \\
QA(SQuAD) & 48.1 & 44.1 & 15.4 & 50.0 & 42.0 & 16.0 & 53.5  & 34.9  & 16.1  & 31.0 & 33.0 & 18.0 \\
\midrule
\textbf{Avg.} & 73.5 & 63.6 & 41.7 & 74.7 & 62.6 & 41.8 & \textbf{74.5} & \textbf{67.2} & \textbf{47.4} & 46.5 & 42.5 & 21.4 \\
\midrule
\textbf{Avg. All} & \multicolumn{3}{c}{59.6} & \multicolumn{3}{c}{59.7} & \multicolumn{3}{c}{\textbf{63.0}} & \multicolumn{3}{c}{36.8} \\
\bottomrule
\end{tabular*}
\caption{Per-task RULER accuracy (\%) on LLaMA2-7B under PI $8\times$ evaluated at 8K, 16K, and 32K. Avg.\ reports the subtask mean at fixed length, while Avg.\ All reports the aggregate score across all tasks and context lengths matching Table~\ref{tab:main_results}.}
\label{tab:ruler_llama2}
\end{table*}

\begin{table*}[t]
\centering
\scriptsize
\renewcommand{\arraystretch}{1.10}
\setlength{\tabcolsep}{3.5pt}
\begin{tabular*}{\textwidth}{@{\extracolsep{\fill}} l ccc ccc ccc @{}}
\toprule
\multirow{2}{*}{\textbf{Task}} & \multicolumn{3}{c}{\textbf{CPT (256M)}} & \multicolumn{3}{c}{\textbf{LongReD (256M)}} & \multicolumn{3}{c}{\textbf{LinearARD (4.25M)}} \\
\cmidrule(lr){2-4} \cmidrule(lr){5-7} \cmidrule(lr){8-10}
 & \textbf{8K} & \textbf{16K} & \textbf{32K} & \textbf{8K} & \textbf{16K} & \textbf{32K} & \textbf{8K} & \textbf{16K} & \textbf{32K} \\
\midrule
NIAH-S1   & 99.0  & 100.0 & 99.0 & 89.0 & 89.0 & 91.0 & 99.0 & 99.0 & 99.0 \\
NIAH-S2   & 100.0 & 100.0 & 99.0 & 86.0 & 84.0 & 62.0 & 96.0 & 96.0 & 88.0 \\
NIAH-S3   & 100.0 & 98.0  & 94.0 & 81.0 & 73.0 & 63.0 & 66.0 & 71.0 & 64.0 \\
NIAH-K1   & 99.0  & 95.0  & 93.0 & 84.0 & 82.0 & 56.0 & 98.0 & 90.0 & 73.0 \\
NIAH-K2   & 99.0  & 80.0  & 60.0 & 92.0 & 83.0 & 66.0 & 95.0 & 96.0 & 61.0 \\
NIAH-K3   & 90.0  & 55.0  & 20.0 & 75.0 & 57.0 & 23.0 & 29.0 & 22.0 & 11.0 \\
NIAH-MQ   & 98.5  & 94.8  & 85.3 & 77.8 & 63.3 & 36.3 & 89.7 & 62.9 & 24.4 \\
NIAH-MV   & 98.8  & 95.3  & 85.0 & 70.5 & 58.8 & 37.0 & 88.4 & 53.0 & 25.8 \\
VT        & 99.8  & 96.2  & 28.6 & 98.0 & 97.0 & 62.2 & 93.1 & 86.6 & 70.0 \\
CWE       & 78.6  & 62.5  & 35.4 & 73.2 & 55.4 & 10.2 & 79.2 & 66.7 & 29.1 \\
FWE       & 88.7  & 85.0  & 72.3 & 80.7 & 85.0 & 66.6 & 67.3 & 70.1 & 58.5 \\
QA(SQuAD) & 53.1  & 60.1  & 29.3 & 49.4 & 59.1 & 26.7 & 56.1 & 47.2 & 18.4 \\
\midrule
\textbf{Avg.} & \textbf{92.0} & \textbf{85.1} & \textbf{66.7} & 79.7 & 73.9 & 50.0 & 79.7 & 71.7 & \textbf{51.9} \\
\midrule
\textbf{Avg. All} & \multicolumn{3}{c}{\textbf{81.3}} & \multicolumn{3}{c}{67.9} & \multicolumn{3}{c}{67.8} \\
\bottomrule
\end{tabular*}
\caption{Per-task RULER accuracy (\%) on LLaMA3-8B under PI $4\times$ evaluated at 8K, 16K, and 32K context lengths.}
\label{tab:ruler_llama3}
\end{table*}

\begin{table*}[t]
\centering
\scriptsize
\renewcommand{\arraystretch}{1.10}
\setlength{\tabcolsep}{3.5pt}
\begin{tabular*}{\textwidth}{@{\extracolsep{\fill}} l ccc ccc ccc @{}}
\toprule
\multirow{2}{*}{\textbf{Task}} & \multicolumn{3}{c}{\textbf{CPT (256M)}} & \multicolumn{3}{c}{\textbf{LongReD (256M)}} & \multicolumn{3}{c}{\textbf{LinearARD (4.25M)}} \\
\cmidrule(lr){2-4} \cmidrule(lr){5-7} \cmidrule(lr){8-10}
 & \textbf{8K} & \textbf{16K} & \textbf{32K} & \textbf{8K} & \textbf{16K} & \textbf{32K} & \textbf{8K} & \textbf{16K} & \textbf{32K} \\
\midrule
NIAH-S1   & 100.0 & 100.0 & 73.0 & 94.0 & 88.0 & 50.0 & 93.0  & 91.0 & 51.0 \\
NIAH-S2   & 100.0 & 100.0 & 60.0 & 91.0 & 90.0 & 63.0 & 96.0  & 89.0 & 64.0 \\
NIAH-S3   & 89.0  & 68.0  & 53.0 & 85.0 & 70.0 & 46.0 & 73.0  & 49.0 & 53.0 \\
NIAH-K1   & 98.0  & 93.0  & 62.0 & 95.0 & 90.0 & 45.0 & 92.0  & 88.0 & 54.0 \\
NIAH-K2   & 80.0  & 40.0  & 3.0  & 96.0 & 93.0 & 49.0 & 94.0  & 92.0 & 54.0 \\
NIAH-K3   & 29.0  & 4.0   & 0.0  & 64.0 & 13.0 & 4.0  & 34.0  & 22.0 & 3.0  \\
NIAH-MQ   & 97.5  & 75.8  & 34.0 & 95.0 & 63.2 & 21.2 & 88.8  & 66.2 & 25.2 \\
NIAH-MV   & 99.5  & 57.2  & 27.5 & 91.0 & 37.2 & 12.2 & 93.7  & 37.3 & 16.8 \\
VT        & 90.2  & 53.4  & 9.4  & 93.4 & 63.4 & 43.4 & 100.0 & 61.0 & 41.2 \\
CWE       & 48.6  & 28.5  & 0.1  & 77.1 & 51.1 & 26.5 & 65.3  & 50.6 & 24.5 \\
FWE       & 7.7   & 72.6  & 20.3 & 79.7 & 92.7 & 59.0 & 72.3  & 82.8 & 40.6 \\
QA(SQuAD) & 58.1  & 40.1  & 17.4 & 52.4 & 40.4 & 16.3 & 49.8  & 40.8 & 17.1 \\
\midrule
\textbf{Avg.} & 74.8 & 61.0 & 30.0 & \textbf{84.5} & \textbf{66.0} & 36.3 & 79.4 & 64.1 & \textbf{37.0} \\
\midrule
\textbf{Avg. All} & \multicolumn{3}{c}{55.3} & \multicolumn{3}{c}{\textbf{62.3}} & \multicolumn{3}{c}{60.2} \\
\bottomrule
\end{tabular*}
\caption{Per-task RULER accuracy (\%) on Mistral-7B-v0.1 under PI $4\times$ evaluated at 8K, 16K, and 32K context lengths.}
\label{tab:ruler_mistral}
\end{table*}

\section{Proof of Proposition~\ref{prop:grad_sensitivity} (Gradient Behavior in Sparse Regimes)}
\label{app:proof_gradient}

\begin{proof}
Let $\mathcal{J} \subseteq \{1, \dots, n\}$ denote the visible key positions for a given query position $i$. For simplicity of notation, let $r_{t,j} \triangleq \mathbf{R}_t(i, j)$ and $r_{s,j} \triangleq \mathbf{R}_s(i, j)$ represent the row-wise teacher and student relation probabilities, where $r_{s,j} = \exp(z_{s,j}) / \sum_{m \in \mathcal{J}} \exp(z_{s,m})$ and $z_{s,j} \triangleq \mathbf{Z}_s(i, j)$.

The Jacobian of the row-wise softmax mapping satisfies:
\begin{equation}
\frac{\partial r_{s,j}}{\partial z_{s,k}} = r_{s,j} (\mathbf{1}_{j=k} - r_{s,k}).
\end{equation}
Applying the chain rule to the row-averaged Mean Squared Error $\mathcal{L}_{\mathrm{MSE}}^{\mathrm{row}} = \frac{1}{2} \sum_{j \in \mathcal{J}} (r_{s,j} - r_{t,j})^2$ with respect to student logit $z_{s,k}$:
\begin{align}
\frac{\partial \mathcal{L}_{\mathrm{MSE}}^{\mathrm{row}}}{\partial z_{s,k}} 
&= \sum_{j \in \mathcal{J}} (r_{s,j} - r_{t,j}) \frac{\partial r_{s,j}}{\partial z_{s,k}} 
= \sum_{j \in \mathcal{J}} (r_{s,j} - r_{t,j}) r_{s,j} (\mathbf{1}_{j=k} - r_{s,k}) \nonumber \\
&= r_{s,k} (r_{s,k} - r_{t,k}) - r_{s,k} \sum_{j \in \mathcal{J}} r_{s,j} (r_{s,j} - r_{t,j}) \nonumber \\
&= r_{s,k} \left[ (r_{s,k} - r_{t,k}) - \sum_{j \in \mathcal{J}} r_{s,j} (r_{s,j} - r_{t,j}) \right]. \label{eq:mse_grad_row}
\end{align}
For forward Kullback-Leibler divergence $\mathcal{L}_{\mathrm{KL}}^{\mathrm{row}} = \sum_{j \in \mathcal{J}} r_{t,j} \log (r_{t,j} / r_{s,j})$, since teacher probabilities $r_{t,j}$ are constant and $\sum_{j \in \mathcal{J}} r_{t,j} = 1$:
\begin{align}
\frac{\partial \mathcal{L}_{\mathrm{KL}}^{\mathrm{row}}}{\partial z_{s,k}} 
&= -\sum_{j \in \mathcal{J}} r_{t,j} \frac{\partial \log r_{s,j}}{\partial z_{s,k}} 
= -\sum_{j \in \mathcal{J}} \frac{r_{t,j}}{r_{s,j}} r_{s,j} (\mathbf{1}_{j=k} - r_{s,k}) \nonumber \\
&= -r_{t,k} + r_{s,k} \sum_{j \in \mathcal{J}} r_{t,j} 
= r_{s,k} - r_{t,k}.
\end{align}
Because $r_{s,j}, r_{t,j} \in [0, 1]$, the bracketed term in Eq.~\eqref{eq:mse_grad_row} is strictly bounded. Thus, in the sparse regime where the teacher places probability mass on a critical token dependency ($r_{t,k} > 0$) but the student assigns vanishing probability ($r_{s,k} \to 0$):
\begin{equation}
\lim_{r_{s,k} \to 0} \frac{\partial \mathcal{L}_{\mathrm{MSE}}^{\mathrm{row}}}{\partial z_{s,k}} = 0, 
\qquad 
\lim_{r_{s,k} \to 0} \frac{\partial \mathcal{L}_{\mathrm{KL}}^{\mathrm{row}}}{\partial z_{s,k}} = -r_{t,k}.
\end{equation}
Averaging over all $n$ query positions scales both gradients by a constant factor of $1/n$, preserving this analytical property.
\end{proof}

\begin{algorithm}[H]
   \caption{Linear-Activation-Memory KL Distillation for Relation Distributions (Backward)}
   \label{alg:shadow_attn}
\begin{algorithmic}[1]
   \STATE {\bfseries Input:} Student $(\mathbf{X}_s, \mathbf{Y}_s)$, Teacher $(\mathbf{X}_t, \mathbf{Y}_t)$.
   \STATE {\bfseries Output:} Gradients $\nabla_{\mathbf{X}_s} \mathcal{L}, \nabla_{\mathbf{Y}_s} \mathcal{L}$.
   \STATE {\bfseries Note:} For QKV self-relations, set $(\mathbf{X}_m, \mathbf{Y}_m) \leftarrow (\mathbf{Q}_m, \mathbf{Q}_m)$, $(\mathbf{K}_m, \mathbf{K}_m)$, or $(\mathbf{V}_m, \mathbf{V}_m)$.
   \STATE \textsc{Phase 1: Global Statistics (Linear HBM Footprint)}
   \STATE $\mathrm{LSE}_s \leftarrow \text{ComputeLSE}(\mathbf{X}_s, \mathbf{Y}_s)$
   \STATE $\mathrm{LSE}_t \leftarrow \text{ComputeLSE}(\mathbf{X}_t, \mathbf{Y}_t)$
   \STATE \textsc{Phase 2: Fused Backward Pass via Tiling}
   \STATE Initialize gradient buffers in HBM: $d\mathbf{X}_s \leftarrow \mathbf{0}, d\mathbf{Y}_s \leftarrow \mathbf{0}$
   \FOR{blocks of queries $\mathbf{X}_s^{(i)}, \mathbf{X}_t^{(i)}$ loaded to SRAM}
       \FOR{blocks of keys $\mathbf{Y}_s^{(j)}, \mathbf{Y}_t^{(j)}$ loaded to SRAM}
           \STATE $\mathbf{Z}_s \leftarrow \frac{1}{\sqrt{d}}\mathbf{X}_s^{(i)} (\mathbf{Y}_s^{(j)})^\top  + \mathbf{M}^{(i,j)}$
           \STATE $\mathbf{Z}_t \leftarrow \frac{1}{\sqrt{d}}\mathbf{X}_t^{(i)} (\mathbf{Y}_t^{(j)})^\top  + \mathbf{M}^{(i,j)}$
           \STATE $\mathbf{R}_s \leftarrow \exp(\mathbf{Z}_s - \mathrm{LSE}_s^{(i)})$
           \STATE $\mathbf{R}_t \leftarrow \exp(\mathbf{Z}_t - \mathrm{LSE}_t^{(i)})$
           \STATE $d\mathbf{Z} \leftarrow (\mathbf{R}_s - \mathbf{R}_t)/n$
           \STATE $d\mathbf{X}_s^{(i)} \mathrel{+}= \frac{1}{\sqrt{d}}\, d\mathbf{Z} \cdot \mathbf{Y}_s^{(j)}$
           \STATE $d\mathbf{Y}_s^{(j)} \mathrel{+}= \frac{1}{\sqrt{d}}\, d\mathbf{Z}^\top \cdot \mathbf{X}_s^{(i)}$
       \ENDFOR
   \ENDFOR
   \STATE \textcolor{gray}{// For self-relations ($\mathbf{X}_s = \mathbf{Y}_s$), total gradient is $d\mathbf{X}_s + d\mathbf{Y}_s$}
\end{algorithmic}
\end{algorithm}

\section{Proof of Theorem~\ref{thm:memory} and Complexity Analysis}
\label{app:complexity}

\begin{proof}
We analyze the High Bandwidth Memory (HBM) consumption of the backward pass across sequence length $n$, batch size $B$, attention heads $H$, and head dimension $d$.

\textbf{Standard Implementation:}
Standard autograd backpropagation retains full intermediate activation maps $\mathbf{R}_s \in \mathbb{R}^{B \times H \times n \times n}$ in HBM to evaluate loss gradients. At $n = 131{,}072$, storing a single full BF16 probability matrix requires $131{,}072^2 \times 2 \approx 34.36\text{ GB}$. Across multiple attention heads, batches, teacher/student copies, and three relation targets ($Q/Q, K/K, V/V$), the required activation storage scales as:
\begin{equation}
\mathcal{M}_{\text{std}} = \mathcal{O}(B H n^2),
\end{equation}
which immediately exceeds the HBM physical capacity of modern enterprise accelerators (e.g., NVIDIA A100/H100 80GB).

\textbf{IO-Aware Tiled Kernel:}
The proposed tiled kernel computes gradients on-the-fly in SRAM across blocks of size $T_r \times T_c$ (e.g., $128 \times 128$). In global HBM, the only retained activation buffers are:
\begin{itemize}
    \item Row-wise log-sum-exp statistics ($\mathrm{LSE}_s, \mathrm{LSE}_t$): Size $2 \times B \times H \times n$.
    \item Accumulated tensor gradients ($d\mathbf{X}, d\mathbf{Y}$): Size $2 \times B \times H \times n \times d$.
\end{itemize}
All intermediate logit tiles $\mathbf{Z}$ and probability distributions $\mathbf{R}$ reside strictly in on-chip SRAM. Consequently, the HBM activation memory footprint is:
\begin{equation}
\mathcal{M}_{\text{LinearARD}} = \mathcal{O}(B H n + B H n d) = \mathcal{O}(B H n d),
\end{equation}
which scales strictly linearly with sequence length $n$ for fixed head dimension $d$.
\end{proof}

\paragraph{Arithmetic Complexity.}
While activation memory is reduced to $\mathcal{O}(n)$, arithmetic operations remain quadratic ($\mathcal{O}(B H n^2 d)$, or $\mathcal{O}(B (H_Q + 2H_{KV}) n^2 d)$ for native GQA heads). Because LinearARD computes mathematically exact full-distribution KL divergence without sparse pruning or low-rank factorizations, both the forward log-sum-exp reduction and the backward recomputation visit every query--key tile block. The kernel renders exact dense relation distillation memory-feasible on ultra-long sequences, while preserving quadratic arithmetic complexity.

\section{Proof of Proposition~\ref{prop:exactness} (Exactness)}
\label{app:exactness}

\begin{proof}
For a fixed batch element and attention head, let:
\begin{equation}
\mathbf{Z}_s = \frac{\mathbf{X}_s \mathbf{Y}_s^\top}{\sqrt{d}} + \mathbf{M}, 
\qquad 
\mathcal{L}_{\mathrm{KL}} = \frac{1}{n} \sum_{i=1}^n D_{\mathrm{KL}}\!\left(\mathbf{R}_t(i, :) \parallel \mathbf{R}_s(i, :)\right).
\end{equation}
Differentiating the row-averaged loss with respect to student logits yields the exact gradient entry:
\begin{equation}
\mathbf{G}(i, j) \equiv \frac{\partial \mathcal{L}_{\mathrm{KL}}}{\partial \mathbf{Z}_s(i, j)} = \frac{\mathbf{R}_s(i, j) - \mathbf{R}_t(i, j)}{n}.
\end{equation}
By the chain rule for the scaled bilinear form, the analytical gradients with respect to factor matrices are:
\begin{equation}
\frac{\partial \mathcal{L}_{\mathrm{KL}}}{\partial \mathbf{X}_s} = \frac{\mathbf{G} \mathbf{Y}_s}{\sqrt{d}}, 
\qquad 
\frac{\partial \mathcal{L}_{\mathrm{KL}}}{\partial \mathbf{Y}_s} = \frac{\mathbf{G}^\top \mathbf{X}_s}{\sqrt{d}}.
\end{equation}
Algorithm~\ref{alg:shadow_attn} reconstructs local tile probabilities via $\mathbf{R}_m(i, j) = \exp(\mathbf{Z}_m(i, j) - \mathrm{LSE}_m(i))$, which shares the identical row-wise normalization as global softmax. For tile block $(I, J)$, the kernel computes local gradient $\mathbf{G}^{I,J}$ and accumulates:
\begin{equation}
d\mathbf{X}_s^I \mathrel{+}= \frac{\mathbf{G}^{I,J} \mathbf{Y}_s^J}{\sqrt{d}}, 
\qquad 
d\mathbf{Y}_s^J \mathrel{+}= \frac{(\mathbf{G}^{I,J})^\top \mathbf{X}_s^I}{\sqrt{d}}.
\end{equation}
By the distributive property of matrix multiplication, summing these partial tile updates across all blocks $J$ and $I$ yields the exact full-matrix products. 

Furthermore, for self-relation targets where $\mathbf{X}_s = \mathbf{Y}_s = \mathbf{U}_s$ (as in $Q/Q, K/K, V/V$), both input branches accumulate into the identical representation tensor:
\begin{equation}
\frac{\partial \mathcal{L}_{\mathrm{KL}}}{\partial \mathbf{U}_s} = \frac{(\mathbf{G} + \mathbf{G}^\top) \mathbf{U}_s}{\sqrt{d}}.
\end{equation}
Summing both accumulated gradient buffers yields exact analytical equivalence to full-matrix backpropagation.
\end{proof}

\section{KL Distillation Backward Kernel}
\label{app:backward_kernel}

Algorithm~\ref{alg:shadow_attn} outlines the fused backward pass for linear-activation-memory relation distillation. During Phase~1, row-wise log-sum-exp vectors $\mathrm{LSE}_s, \mathrm{LSE}_t \in \mathbb{R}^n$ are precomputed. Phase~2 executes tiled recomputation in SRAM without materializing full $n \times n$ matrices in HBM. For self-relation distillation ($Q/Q, K/K, V/V$), gradients flowing through both query and key argument roles are accumulated to form the full symmetric gradient.

\section{Supplementary Analyses and Numerical Verification}
\label{app:supplementary_analysis}

\subsection{Attention--Logit Coupling}
\label{subsubsec:attn_logit_coupling}

\begin{figure}[H]
    \centering
    \maybeincludegraphics[width=0.69\linewidth]{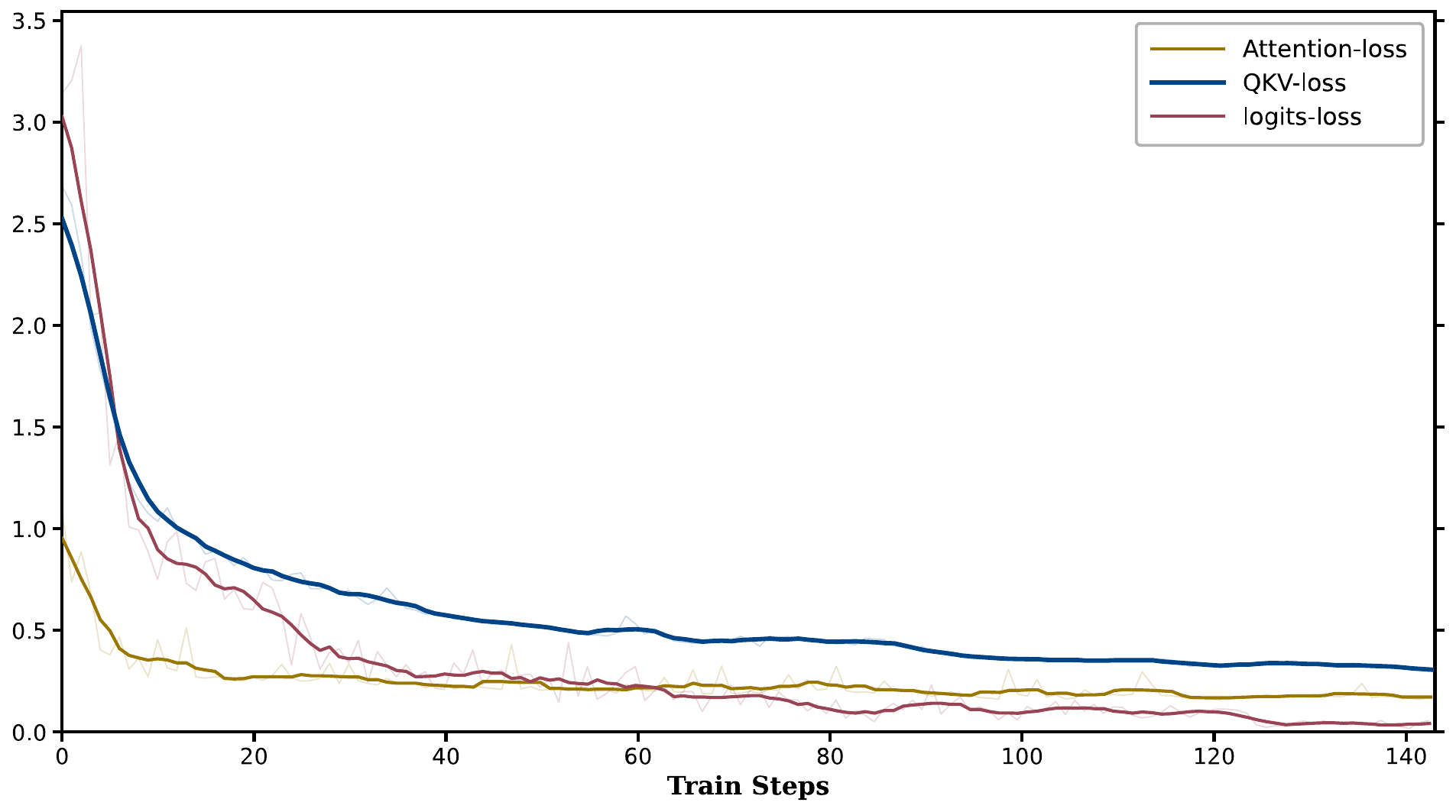}
    \caption{Training dynamics of distillation objectives. We optimize relation alignment (\textbf{QKV-loss}, blue) while monitoring unoptimized metrics (\textbf{Attention-loss}, gold; \textbf{Logit-loss}, maroon). Their concurrent decline provides strong empirical evidence that output distribution drift is closely associated with internal attention geometry distortion.}
    \label{fig:training_curves}
\end{figure}

To investigate the internal restoration process, Figure~\ref{fig:training_curves} monitors intermediate representations during training. When optimizing exclusively on self-relation distributions (\textbf{QKV-loss}, decreasing from 2.69 to 0.30), unoptimized metrics decline concurrently: the standard attention-map loss drops from 1.09 to 0.16, and the output-logit KL collapses from 3.14 to approximately 0.05. This demonstrates a strong empirical coupling between internal relational geometry and output distribution alignment. Rectifying internal relational distortions within the attention mechanism enables the model to recover native output behavior without requiring explicit vocabulary logit supervision.

\subsection{Kernel Numerical Verification}
\label{subsubsec:kernel_numerical_error}

\begin{table}[H]
\centering
\footnotesize
\renewcommand{\arraystretch}{1.15}
\setlength{\tabcolsep}{0pt}
\begin{tabular*}{\linewidth}{@{\extracolsep{\fill}} r S[table-format=1.2] S[table-format=1.1] S[table-format=1.1] @{}}
\toprule
\multirow{2}{*}{\textbf{Length}} & {\textbf{Forward Loss Error}} & \multicolumn{2}{c}{\textbf{Backward Gradient Error}} \\
\cmidrule(lr){2-2} \cmidrule(l){3-4}
 & {Relative MAE ($\times 10^{-7}$)} & {Mean Relative ($\times 10^{-4}$)} & {Max Relative ($\times 10^{-2}$)} \\
\midrule
256  & 4.9 & 1.8 & 0.6 \\
512  & 4.9 & 1.7 & 0.7 \\
1024 & 4.7 & 1.5 & 0.8 \\
2048 & 4.6 & 1.2 & 0.9 \\
4096 & 4.9 & 1.0 & 1.0 \\
\bottomrule
\end{tabular*}
\caption{Kernel numerical verification against a materialized dense baseline. Forward loss relative MAE is measured in FP32; backward gradient errors are evaluated in BF16 normalized by mean tensor magnitude.}
\label{tab:linear_kl_error}
\end{table}

While Proposition~\ref{prop:exactness} proves analytical equivalence, block tiled reductions alter floating-point summation associativity. Table~\ref{tab:linear_kl_error} quantifies numerical deviations against a materialized dense implementation across sequence lengths. The forward FP32 loss deviation remains bounded at $10^{-7}$ (on the order of machine epsilon $u_{32} = 2^{-24} \approx 5.96 \times 10^{-8}$), while BF16 backward gradient deviations remain strictly governed by the 7-bit fraction precision floor ($u_{b16} = 2^{-8} \approx 3.9 \times 10^{-3}$). These empirical measurements confirm that the tiled kernel achieves numerical agreement up to target floating-point rounding precision.

\section{LongBench-E Protocol and Detailed Results}
\label{app:longbench_protocol}

To evaluate broad long-context comprehension beyond synthetic retrieval, we benchmark models on 13 tasks from LongBench-E within the 32K context window: 2WikiMQA, GovReport, HotpotQA, LCC, MultiNews, MultiFieldQA-en, PassageCount, PassageRetrieval-en, Qasper, RepoBench-P, SAMSum, TREC, and TriviaQA. This suite spans single- and multi-document question answering, document summarization, code completion, passage retrieval and counting, and classification.

Following standard practice, evaluations are conducted zero-shot using official datasets. Input prompts are tokenized with each model's native tokenizer, and examples are retained when the prompt length plus the official maximum generation allowance fits within the 32,768-token limit. Table~\ref{tab:longbench_retention} details sample retention for LLaMA2-7B, where 3,572 of 3,668 examples (97.38\%) are evaluated. Performance is macro-averaged across tasks and grouped into three input-length intervals: below 4,000 tokens, 4,000--7,999 tokens, and 8,000+ tokens.

\begin{table}[H]
\centering
\small
\renewcommand{\arraystretch}{1.08}
\setlength{\tabcolsep}{0pt}
\begin{tabular*}{\linewidth}{@{\extracolsep{\fill}} l r r r @{}}
\toprule
\textbf{Task} & \textbf{Total Examples} & \textbf{Evaluated (32K)} & \textbf{Excluded} \\
\midrule
2WikiMQA             & 300  & 300  & 0  \\
GovReport            & 300  & 299  & 1  \\
HotpotQA             & 300  & 300  & 0  \\
LCC                  & 300  & 269  & 31 \\
MultiNews            & 294  & 289  & 5  \\
MultiFieldQA-en      & 150  & 150  & 0  \\
PassageCount         & 300  & 300  & 0  \\
PassageRetrieval-en  & 300  & 300  & 0  \\
Qasper               & 224  & 224  & 0  \\
RepoBench-P          & 300  & 242  & 58 \\
SAMSum               & 300  & 300  & 0  \\
TREC                 & 300  & 300  & 0  \\
TriviaQA             & 300  & 299  & 1  \\
\midrule
\textbf{Total}       & 3668 & 3572 & 96 \\
\bottomrule
\end{tabular*}
\caption{LongBench-E sample retention under the 32,768-token total sequence length limit for LLaMA2-7B.}
\label{tab:longbench_retention}
\end{table}

\begin{table}[H]
\centering
\small
\renewcommand{\arraystretch}{1.10}
\setlength{\tabcolsep}{0pt}
\begin{tabular*}{\linewidth}{@{\extracolsep{\fill}} l l l r r r r @{}}
\toprule
\textbf{Scaling} & \textbf{Method} & \textbf{Tokens} & \textbf{Overall} & \textbf{$<$4K} & \textbf{4--8K} & \textbf{8K+} \\
\midrule
PI  & Direct CPT & 4.25M & 13.22 & 10.25 & 13.94 & 13.21 \\
PI  & LinearARD  & 4.25M & \textbf{22.65} & \textbf{22.62} & \textbf{22.50} & \textbf{22.82} \\
PI  & Direct CPT & 256M  & 22.94 & 22.79 & 23.78 & 22.82 \\
PI  & LinearARD  & 256M  & \textbf{23.80} & \textbf{26.19} & \textbf{23.86} & \textbf{23.72} \\
\midrule
NTK & Direct CPT & 4.25M & 21.36 & 21.95 & \textbf{24.33} & 19.91 \\
NTK & LinearARD  & 4.25M & \textbf{22.11} & \textbf{23.05} & 22.98 & \textbf{21.43} \\
NTK & Direct CPT & 256M  & 22.96 & 22.45 & \textbf{24.59} & 21.97 \\
NTK & LinearARD  & 256M  & \textbf{23.19} & \textbf{24.96} & 24.39 & \textbf{22.03} \\
\bottomrule
\end{tabular*}
\caption{LLaMA2-7B LongBench-E performance across input-length intervals under PI and NTK-aware scaling.}
\label{tab:longbench_lengths_llama2}
\end{table}

\begin{table}[H]
\centering
\small
\renewcommand{\arraystretch}{1.10}
\setlength{\tabcolsep}{0pt}
\begin{tabular*}{\linewidth}{@{\extracolsep{\fill}} l l l r r r r @{}}
\toprule
\textbf{Backbone} & \textbf{Method} & \textbf{Tokens} & \textbf{Overall} & \textbf{$<$4K} & \textbf{4--8K} & \textbf{8K+} \\
\midrule
\multirow{4}{*}{LLaMA3-8B} 
 & Direct CPT & 4.25M & 12.799 & 13.584 & 13.385 & 12.745 \\
 & LinearARD  & 4.25M & \textbf{21.184} & \textbf{22.245} & \textbf{21.658} & \textbf{20.805} \\
 & Direct CPT & 256M  & 21.392 & 20.522 & \textbf{22.816} & 21.179 \\
 & LinearARD  & 256M  & \textbf{22.799} & \textbf{22.942} & 22.433 & \textbf{23.024} \\
\midrule
\multirow{4}{*}{Mistral-7B-v0.1} 
 & Direct CPT & 4.25M & 4.780  & 5.443  & 5.079  & 4.443  \\
 & LinearARD  & 4.25M & \textbf{17.167} & \textbf{15.861} & \textbf{16.439} & \textbf{17.440} \\
 & Direct CPT & 256M  & 20.663 & 17.877 & 21.169 & 20.689 \\
 & LinearARD  & 256M  & \textbf{21.796} & \textbf{18.965} & \textbf{22.885} & \textbf{21.415} \\
\bottomrule
\end{tabular*}
\caption{LongBench-E performance across input-length intervals on GQA backbones extended from 8K to 32K ($4\times$ PI).}
\label{tab:longbench_lengths_gqa}
\end{table}

Table~\ref{tab:longbench_lengths_llama2} and Table~\ref{tab:longbench_lengths_gqa} present length-grouped evaluations. Across all three model architectures, 4.25M LinearARD delivers substantial gains over budget-matched CPT (LLaMA2-7B: 22.65 vs.\ 13.22; LLaMA3-8B: 21.18 vs.\ 12.80; Mistral-7B: 17.17 vs.\ 4.78), effectively matching or exceeding 256M direct CPT baselines with approximately $60\times$ fewer tokens.

\subsection{LLaMA2-7B Task-Level Performance Breakdowns}

Table~\ref{tab:longbench_tasks_llama2} consolidates task-level performance across all 13 benchmarks for LLaMA2-7B. Under the 4.25M budget, LinearARD achieves remarkable recoveries on tasks severely degraded by linear interpolation, notably surging on TriviaQA ($12.35 \to 68.44$), TREC ($13.33 \to 55.67$), and SAMSum ($4.72 \to 22.03$).

\begin{table*}[t]
\centering
\small
\renewcommand{\arraystretch}{1.10}
\setlength{\tabcolsep}{3.5pt}
\begin{tabular*}{\textwidth}{@{\extracolsep{\fill}} l rrrr rrrr @{}}
\toprule
\multirow{3}{*}{\textbf{Task}} 
 & \multicolumn{4}{c}{\textbf{Position Interpolation (PI $8\times$)}} 
 & \multicolumn{4}{c}{\textbf{NTK-Aware Scaling (NTK $8\times$)}} \\
\cmidrule(lr){2-5} \cmidrule(lr){6-9}
 & \multicolumn{2}{c}{\textbf{4.25M Budget}} & \multicolumn{2}{c}{\textbf{256M Budget}}
 & \multicolumn{2}{c}{\textbf{4.25M Budget}} & \multicolumn{2}{c}{\textbf{256M Budget}} \\
\cmidrule(lr){2-3} \cmidrule(lr){4-5} \cmidrule(lr){6-7} \cmidrule(lr){8-9}
 & \textbf{CPT} & \textbf{LinearARD} & \textbf{CPT} & \textbf{LinearARD} & \textbf{CPT} & \textbf{LinearARD} & \textbf{CPT} & \textbf{LinearARD} \\
\midrule
2WikiMQA            &  4.91 &  8.59 &  8.14 &  7.90 &  8.85 &  8.01 &  8.04 &  7.21 \\
GovReport           & 22.75 & 25.24 & 28.14 & 26.60 & 26.72 & 25.87 & 24.26 & 26.31 \\
HotpotQA            &  6.18 &  8.26 &  8.43 &  8.35 &  7.35 &  8.44 &  8.75 &  8.49 \\
LCC                 & 34.54 & 26.83 & 29.54 & 31.86 & 33.07 & 28.25 & 23.09 & 27.05 \\
MultiNews           & 16.90 & 21.01 & 20.36 & 19.92 & 18.74 & 18.26 & 19.24 & 19.95 \\
MultiFieldQA-en     & 11.56 & 15.72 & 15.64 & 16.44 & 12.86 & 15.23 & 14.54 & 15.63 \\
PassageCount        &  4.25 &  4.43 &  3.65 &  4.18 &  3.86 &  2.69 &  5.37 &  3.26 \\
PassageRetrieval-en &  4.19 &  5.55 &  6.77 &  8.14 &  4.03 &  3.44 &  6.30 &  6.69 \\
Qasper              &  6.58 &  8.82 &  9.45 &  9.09 &  7.47 &  8.52 &  8.62 &  9.68 \\
RepoBench-P         & 29.67 & 23.86 & 21.77 & 21.65 & 28.59 & 26.98 & 18.60 & 19.33 \\
SAMSum              &  4.72 & 22.03 & 21.97 & 22.76 & 29.18 & 21.15 & 29.78 & 26.21 \\
TREC                & 13.33 & 55.67 & 52.44 & 57.33 & 49.17 & 54.44 & 59.17 & 57.00 \\
TriviaQA            & 12.35 & 68.44 & 71.91 & 75.16 & 47.74 & 66.10 & 72.74 & 74.66 \\
\midrule
\textbf{Macro Avg}  & \textbf{13.22} & \textbf{22.65} & \textbf{22.94} & \textbf{23.80} & \textbf{21.36} & \textbf{22.11} & \textbf{22.96} & \textbf{23.19} \\
\bottomrule
\end{tabular*}
\caption{Task-level LongBench-E evaluation on LLaMA2-7B across positional scaling configurations, training methods, and token budgets. LinearARD substantially improves accuracy on retrieval, question answering, and reasoning tasks under constrained budgets while maintaining competitive performance at scale.}
\label{tab:longbench_tasks_llama2}%
\label{tab:longbench_tasks_1}% 绑定原有label，保证正文任何可能存在的索引不失效
\end{table*}

\section{PI and Fixed-Theta NTK Scaling}
\label{app:scaling_transfer}

To evaluate whether LinearARD transfers beyond linear Position Interpolation (PI), we benchmark against fixed-$\theta$ NTK-aware scaling on LLaMA2-7B extended from 4K to 32K context ($8\times$ extension). 

\paragraph{Positional Configuration.}
For PI scaling, we apply linear position compression with factor 8.0 and maximum sequence length 32,768, scaling token position indices as $p/8$ while keeping base frequency $\theta_0 = 10{,}000$ constant. For fixed-$\theta$ NTK scaling, position indices remain unscaled while the rotary base frequency is adjusted statically based on head dimension $d = 128$:
\begin{equation}
\theta' = \theta_0 \times 8^{\frac{d}{d - 2}} = 10{,}000 \times 8^{\frac{128}{126}} \approx 82{,}684.6226.
\end{equation}
The maximum position embedding is configured to 32,768. The designated positional configuration is held constant throughout relation distillation, adapter merging, continuation training, and evaluation. The native teacher model strictly maintains unscaled RoPE.

\begin{table}[H]
\centering
\small
\renewcommand{\arraystretch}{1.10}
\setlength{\tabcolsep}{0pt}
\begin{tabular*}{\linewidth}{@{\extracolsep{\fill}} l l l r r r @{}}
\toprule
\textbf{Scaling Method} & \textbf{Training Method} & \textbf{Tokens} & \textbf{Short Avg.} & \textbf{LongBench-E} & \textbf{RULER} \\
\midrule
\multirow{4}{*}{PI $8\times$} 
 & Direct CPT & 4.25M & 41.03 & 13.22 & 0.040 \\
 & LinearARD  & 4.25M & \textbf{50.93} & \textbf{22.65} & \textbf{63.00} \\
 & Direct CPT & 256M  & 52.20 & 22.94 & 59.60 \\
 & LinearARD  & 256M  & \textbf{53.38} & \textbf{23.80} & \textbf{65.10} \\
\midrule
\multirow{4}{*}{NTK $8\times$} 
 & Direct CPT & 4.25M & 52.37 & 21.36 & 32.85 \\
 & LinearARD  & 4.25M & \textbf{53.75} & \textbf{22.11} & \textbf{64.84} \\
 & Direct CPT & 256M  & 53.40 & 22.96 & 64.53 \\
 & LinearARD  & 256M  & \textbf{53.99} & \textbf{23.19} & \textbf{67.97} \\
\bottomrule
\end{tabular*}
\caption{Comprehensive performance comparison between Position Interpolation (PI) and fixed-$\theta$ NTK-aware scaling on LLaMA2-7B across token budgets. Short Avg.\ reports the 8-benchmark short-context mean, LongBench-E denotes macro-averaged task accuracy on the 32K-compatible subset, and RULER reports aggregate retrieval accuracy across 8K, 16K, and 32K context windows.}
\label{tab:pi_ntk}
\end{table}

\paragraph{Analysis of Scaling Transfer.}
As reported in Table~\ref{tab:pi_ntk}, unscaled models modified with NTK scaling retain higher initial short-context accuracy than PI (51.66 vs.\ 34.23 without restoration) due to preserving high-frequency rotary dynamics. Despite these distinct mathematical starting points, LinearARD consistently improves performance across both scaling methods and training budgets. Under 4.25M tokens, LinearARD lifts NTK RULER performance from 32.85 to 64.84, while simultaneously enhancing LongBench-E from 21.36 to 22.11 and short-context accuracy from 52.37 to 53.75. At the 256M budget, LinearARD maintains consistent advantages across all benchmarks. These results confirm that internal attention relation restoration generalizes robustly across materially distinct RoPE scaling schedules.

\section{Broader Impacts}
\label{app:broader_impacts}

LinearARD's primary broader impact is to make long-context adaptation more efficient and reproducible. By reducing the memory and token budget required to restore RoPE-scaled LLMs, the method can lower compute and energy costs, broaden access for academic groups and smaller organizations, and support useful applications such as long-document assistance, retrieval-augmented reasoning, technical-document analysis, and educational or scientific workflows. The main risks are those inherited from adapting capable base LLMs: improved efficiency could also be used in harmful deployments, and restored models may retain biases, privacy risks, or security weaknesses from the teacher model. Responsible use should follow the licenses and safeguards of the underlying models and data, avoid sensitive or personally identifiable training data without an appropriate basis, and include downstream checks for bias, toxicity, privacy leakage, and unsafe generation. This work does not release a new high-risk pretrained model, image generator, or scraped dataset.

\end{document}